\documentclass[a4paper,fleqn]{cas-dc}

\usepackage{natbib}
\usepackage{amsmath,amsfonts}
\usepackage{algorithmic}
\usepackage{algorithm}
\usepackage{array}
\usepackage{color}
\usepackage{textcomp}
\usepackage{stfloats}
\usepackage{url}
\usepackage{verbatim}
\usepackage{graphics} 
\usepackage{graphicx}
\usepackage{epsfig} 
\usepackage{times} 
\usepackage{amssymb} 
\usepackage{amsthm}
 \newtheorem{theorem}{Theorem}
 
 \newdefinition{rmk}{Remark}
 \newproof{pf}{Proof}
 \newproof{pot}{Proof of Theorem \ref{thm2}}
\usepackage{booktabs}
\usepackage{subcaption}
\usepackage{overpic}

\def\tsc#1{\csdef{#1}{\textsc{\lowercase{#1}}\xspace}}
\tsc{WGM}
\tsc{QE}
\tsc{EP}
\tsc{PMS}
\tsc{BEC}
\tsc{DE}

\begin{document}
\let\WriteBookmarks\relax
\def\floatpagepagefraction{1}
\def\textpagefraction{.001}
\shorttitle{Dynamic Motion/Force Control of Mobile Manipulators via Extended UDE
}
\shortauthors{S. Gao et~al.}

\title [mode = title]{Dynamic Motion/Force Control of Mobile Manipulators via Extended UDE
}


\author[1]{Songqun Gao} [orcid=0000-0002-2434-8656]
\credit{Writing – original draft, Methodology, Software, Validation}

\author[1]{Wendi Ding}
\credit{Writing - review \& editing, Investigation, Methodology, Validation}

\author[2]{Maotong Cheng}
\credit{Writing - review \& editing, Methodology, Software, Validation}

\author[2]{Qinyuan Ren}
\cormark[1]
\credit{Supervision, Writing - review \& editing, Resources}
\ead{renqinyuan@zju.edu.cn}

\author[1]{Ben M. Chen}
\credit{Supervision, Resources, Conceptualization}

\affiliation[1]{organization={Department of Mechanical and Automation Engineering, The Chinese University of Hong Kong},
                city={Hong Kong SAR},
                postcode={999077},
                country={China}
                }

\affiliation[2]{organization={College of Control Science and Engineering, Zhejiang University},
                city={Hangzhou},
                postcode={310027}, 
                country={China}}

\cortext[cor1]{Corresponding author}


\begin{abstract}
Mobile manipulators are known for their superior mobility over manipulators on fixed bases, offering promising applications in smart industry and housekeeping scenarios. The dynamic coupling nature between the mobile base and the manipulator presents challenges for force interactive tasks of the mobile manipulator. However, current strategies often fail to account for this coupling in such scenarios. To address this, this paper presents a dynamic coupling-integrated manipulator model that requires only the manipulator dynamics and the mobile base kinematics, which simplifies the modeling process. In addition, embedding the dynamic model, an extended uncertainty and disturbance estimator (UDE) is proposed for the mobile manipulator, which separately estimates the dynamic coupling terms and other unmodeled uncertainties, incorporating them into the feedforward and feedback control loops, respectively. The proposed approach increases the speed of response of the system and improves the dynamic robot-environment interaction (REI) performance of the mobile manipulator. A series of simulations and experiments of a wall-cleaning task are conducted to verify the effectiveness of the proposed approach. Ablation studies demonstrate that the proposed approach significantly improves the motion/force tracking performance when the mobile base is in dynamic motion. 
\end{abstract}


\maketitle

\section{Introduction}
The integration of manipulators on mobile bases, namely mobile manipulators, has recently gained popularity due to their excellent mobility and capabilities to facilitate large-scale interactions. Mobile manipulators have been adopted in various applications, such as housekeeping \citep{wall1,humaninteract}, industrial inspection \citep{aerialinspection1, aerialinspection2}, leak inspection \citep{mm_inspection}, painting \citep{painting}, collaborative transport \citep{coop1}, and so on.
In these applications, mobile manipulators are required to handle the coupled motion of mobile bases and manipulators to guarantee efficient and safe operations in complex environments \citep{gupta_timeoptimal, mm_coupling}. 
However, due to the rapid acceleration, deceleration, turning of the mobile base, and the unevenness of the ground, the mobile base transmits the nonlinear dynamic coupling and disturbances to the onboard manipulator, which generate undesired movement, leading to robot-environment interaction (REI) performance degradations \citep{mineexplore1} and potential collisions \citep{gupta_uncertain, aerial_compliant}. Therefore, ensuring stable, accurate REI of mobile manipulators while the mobile base is in high-dynamic motion remains a challenge.

Several studies on mobile manipulators compensate for dynamic coupling terms using model-based control approaches. 
The dynamic coupling terms have been modeled as part of onboard manipulator dynamics \citep{aerialmanipulate8}, whole-body dynamics \citep{iterlearn_1}, and variable inertia parameters of manipulators \citep{aerialmanipulate7}. However, model-based control relies heavily on the accuracy of the dynamic coupling model, which is difficult to obtain in real applications. To address this limitation, observers for uncertainty rejection become a promising solution when the dynamics of manipulators/mobile manipulators are partially known, for example, the disturbance estimator \citep{aerialmanipulate7}, adaptive neural network \citep{mm1}, dynamic compensator \citep{dc_cep}, disturbance observer \citep{wall1}, uncertainty and disturbance estimator (UDE) \citep{UDE19}, and so on. Observers can be integrated with robust control schemes, where the observer compensates for the unmodeled dynamics and the robust controller guarantees the stability of the system against disturbances. 
Moreover, observers can be further extended by considering the specific characteristics of uncertainties. For instance, UDE has been adapted for handling time-varying disturbances \citep{UDE_TV}, periodic disturbances \citep{UDE_wind_g}, and time-delay systems \citep{UDE_timedelay}. 
In addition, disturbance observers have also achieved effective results in dealing with system uncertainties and dynamic disturbances \citep{dob}. 
However, these strategies suffer from phase lag, which hinders effective real-time compensation for the dynamic coupling of mobile manipulators. 
To solve this, in \citep{mineexplore1}, to reduce material spillage of forklifts when operating on uneven terrain, a feedforward term is integrated with an $H_\infty$ controller to increase the transient response of the compliant forklift against the terrain-induced disturbance. Nevertheless, feedforward control is sensitive to noise and the dynamic coupling of mobile manipulators requires acceleration information from mobile bases, which is usually unavailable. 
Although the dynamic coupling of mobile manipulators has been extensively studied, ensuring stable and accurate REI of mobile manipulators under high-dynamic motion remains challenging.

In this paper, a control approach of extended UDE is proposed for dynamic interaction tasks as illustrated in Fig.\ref{mmplatform} (a) to deal with the dynamic coupling of mobile manipulators and improve the transient response of the manipulator under the high-dynamic motion of mobile bases. 
First, dynamic parameters of the mobile base, such as rotational inertia and friction coefficients, are not only difficult to measure but also vary with environment conditions. To solve this, a dynamic coupling-integrated model is developed, incorporating base kinematics into manipulator dynamics. This simplified model facilitates the design and implementation of the extended UDE and improves efficiency by avoiding the tedious modeling of whole-body dynamics. 
Second, to improve the transient response of mobile manipulators under dynamic couplings, an extended UDE is designed to estimate dynamic coupling terms and other unmodeled uncertainties and incorporate them into the feedforward and feedback control loops, respectively. The extended UDE requires only the velocity information of the mobile base, thereby eliminating the need for acceleration measurements. 
Finally, the proposed approach is applied to a wall-cleaning task, which is a typical application of dynamic interaction. Experimental results reveal that the proposed approach effectively estimates the dynamic coupling terms and improves the stability and accuracy of dynamic interaction. 

The contributions of this paper are concluded as follows:
\begin{enumerate}[\textbullet]
\item A dynamic coupling-integrated manipulator model is proposed in this study, which streamlines the modeling process for control purposes and avoids the tedious whole-body dynamic modeling of mobile manipulators.
\item Incorporated with the dynamic coupling model, an extended UDE is designed to improve the transient response of the end effector under the high-dynamic motion of the mobile base. 
\item The proposed approach is applied to a typical wall-cleaning task, and the effectiveness of the proposed method in compensating for dynamic coupling terms and uncertainties is verified through experiments. Furthermore, the limitation of the approach in compliant environment interaction is analyzed.

\end{enumerate}

The rest of the paper is organized as follows:
Sec. II formulates the dynamic coupling-integrated manipulator model. Sec. III introduces the proposed dynamic coupling-integrated control approach through an extended UDE. Sec. IV demonstrates the effectiveness of the proposed control approach through simulations and experiments. Finally, Sec. V concludes this paper. 

\begin{figure}[!t]
	\centering
        \includegraphics[width=0.95\linewidth]{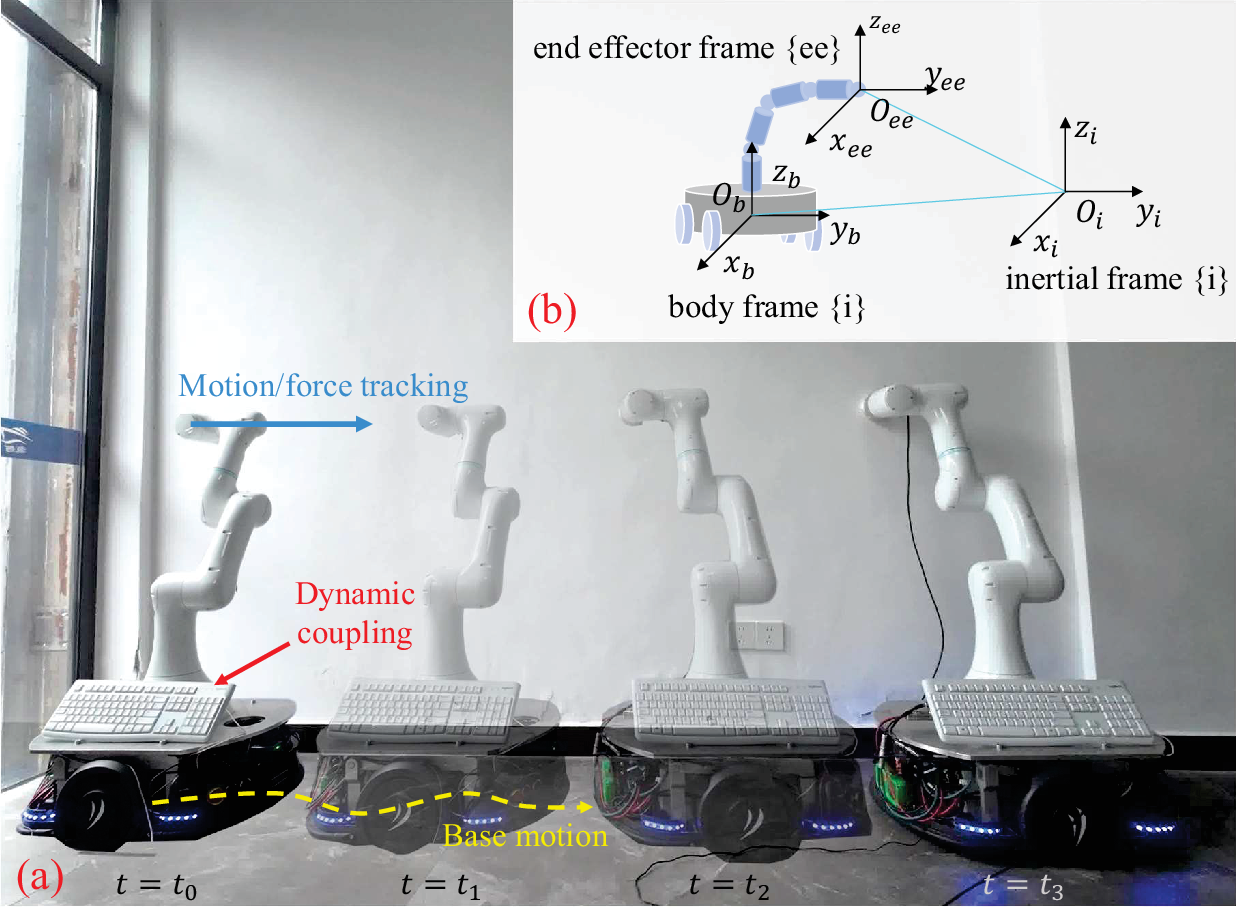}
	\caption{(a) Composite images of the wall-cleaning demonstration of the mobile manipulator at different time steps: $t_0,t_1,t_2$ and $t_3$. (b) The mobile manipulator platform is described using three frames of reference under the Cartesian coordinate system: the inertial frame $\{i\}$, the body frame $\{b\}$, and the end effector frame $\{ee\}$. }
	\label{mmplatform}
\end{figure}

\section{Dynamic Modeling of Manipulator on the Mobile Base}
\noindent Figure~\ref{mmplatform} (b) shows an $n$-degree of freedom (DOF) manipulator attached to a moving base. The body frame $\{b\}$ centers at the center of gravity (CG) of the mobile base, denoted by $O_b$, and the end effector frame $\{ee\}$ centers at the CG of the end effector of the manipulator, namely $O_{ee}$. The dynamic equation of the manipulator can be represented as:
\begin{equation}
    \begin{aligned}
    M_q \boldsymbol{\Ddot{q}} + C_q \boldsymbol{\dot{q}} + G_q =  \boldsymbol{\tau} + \hat{J}^T (\boldsymbol{f_e + f_d}), 
    \end{aligned} 
    \label{dynamics2}
\end{equation}
where $\boldsymbol{q} \in \mathbb{R}^n$ denotes the joint vector of the manipulator. $M_q$, $C_q$, $G_q$ represent the inertial, the Coriolis and centrifugal forces, and the gravity system matrices. $\boldsymbol{\tau} \in \mathbb{R}^n$ represents the input torque vector exerted on the joint. $\boldsymbol{f_e} = [f_x, f_y, f_z, f_m, f_n, f_k]^T\in \mathbb{R}^6$ and $f_d$ are expressed in $\{i\}$ and account for wrenches and disturbances due to contact with the environment. $\hat{J} =  \begin{bmatrix} {R^i_b} & \mathbf{0} \\ \mathbf{0} & {R^i_b} \end{bmatrix} J$ is the augmented Jacobian matrix, where $J$ represents the analytic Jacobian matrix of the manipulator and $R^i_b$ denotes the rotation matrix of frame $\{i\}$ with respect to $\{b\}$. 
The Cartesian coordinates of the mobile manipulator expressed in the inertial frame $\{i\}$ 
are given by $\boldsymbol{X} = \begin{bmatrix} \boldsymbol{\eta} \\ \boldsymbol{x} \end{bmatrix} \in \mathbb{R}^{12}$, where $\boldsymbol{\eta} = \begin{bmatrix} \boldsymbol{p_\eta} \\ \boldsymbol{\Theta_\eta} \end{bmatrix} = [\eta_{1}, \eta_{2}, \eta_{3}, \eta_{4}, \eta_{5}, \eta_{6}]^T \in \mathbb{R}^{6}$ are the Cartesian coordinates of the mobile base and $\boldsymbol{x} = \begin{bmatrix} \boldsymbol{p_x} \\ \boldsymbol{\Theta_x} \end{bmatrix} = [x_{1}, x_{2}, x_{3}, x_{4}, x_{5}, x_{6}]^T \in \mathbb{R}^{6}$ are the Cartesian coordinates of the end effector of the manipulator.

\begin{figure*}[!ht]
	\centering
	\includegraphics[scale=0.5]{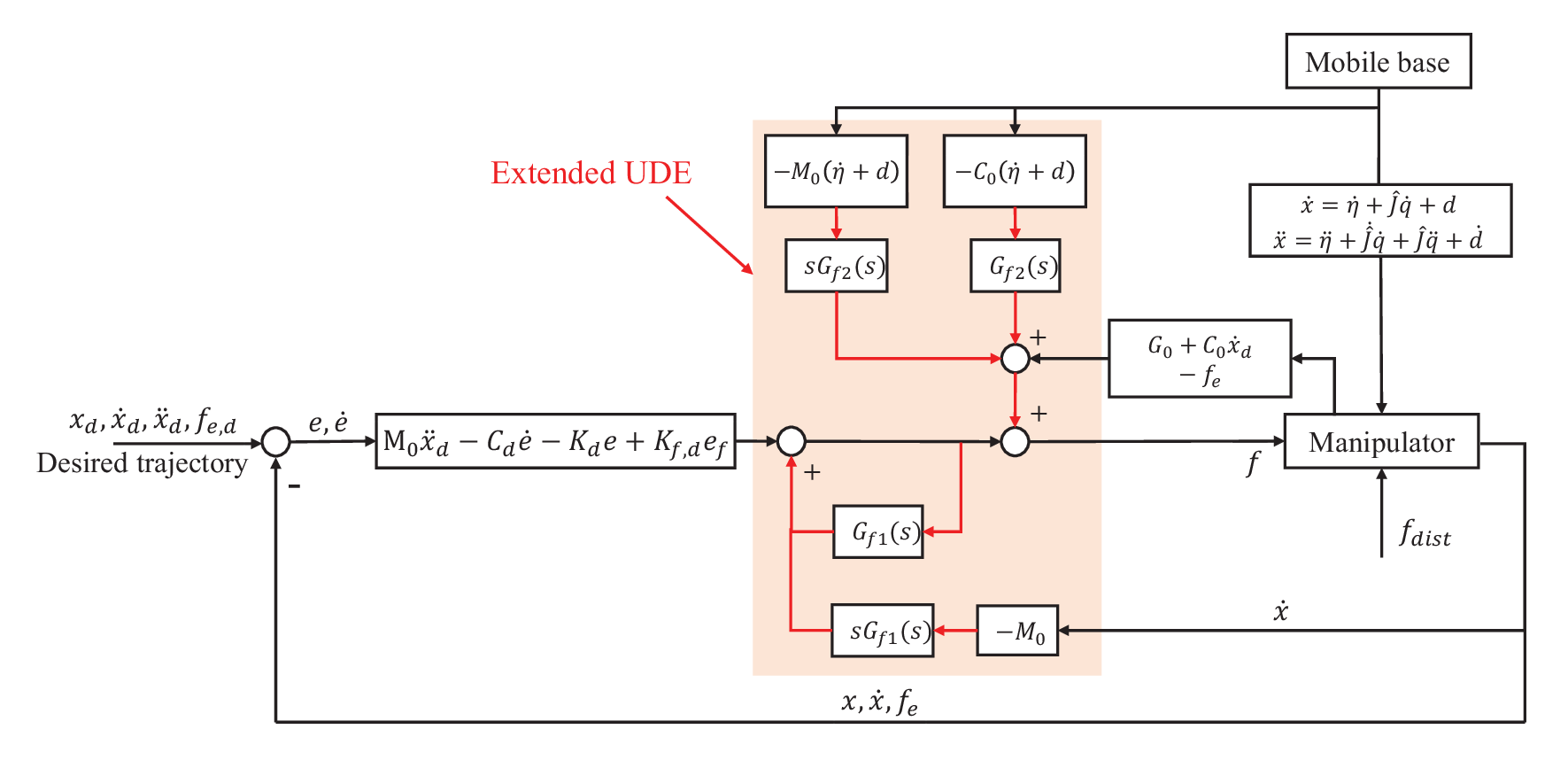}
	\caption{Control diagram of the dynamic motion/force control of mobile manipulators through extended UDE. In the extended UDE, the dynamic coupling term is compensated in the feedforward control loop, while other unmodeled system uncertainties are compensated in the feedback control loops. }
	\label{control flow}
\end{figure*}

Let $\phi(t; 0, \boldsymbol{x_d}(0),  \boldsymbol{\dot{x}_d}(0), \boldsymbol{f_{e,d}}(0))$ represent the desired trajectory of the end effector of the manipulator during the time interval $[0, t]$ with an initial state given by $\boldsymbol{x_d}(0), \boldsymbol{\dot{x}_d}(0), \boldsymbol{f_{e,d}}(0)$. The task of the mobile manipulator is to track the motion/force trajectory $\phi(t; 0, \boldsymbol{x_d}(0), \boldsymbol{\dot{x}_d}(0)$, $\boldsymbol{f_{e,d}}(0))$ in Cartesian space for the end effector while maintaining the compliance of the end effector and compensating dynamic coupling terms and other unmodeled uncertainties during operation.

According to Fig. \ref{mmplatform} (b), the kinematic relation between $\eta$ and $x$ is derived as:
\begin{equation}
\begin{aligned}
\boldsymbol{p_x} = \boldsymbol{p_\eta} +& \boldsymbol{P^i_{ee/b}} = \boldsymbol{p_\eta} + R^i_b \boldsymbol{P^b_{ee/b}}, \\
\boldsymbol{R^i_{ee}} &= R^i_b \boldsymbol{R^b_{ee}},
\label{kine1}
\end{aligned} 
\end{equation}
where $R^i_b$, $R^b_{ee}$, $R^i_{ee}$ are the rotation matrices of frame $\{i\}$ with respect to $\{b\}$, $\{b\}$ with respect to $\{ee\}$, $\{i\}$ with respect to $\{ee\}$. $\boldsymbol{P^b_{ee/b}}$ denotes the position of the end effector with respect to the body frame $\{b\}$ expressed in $\{b\}$ and $\boldsymbol{P^i_{ee/b}}$ denotes the position of the end effector with respect to the body frame $\{b\}$ expressed in $\{i\}$. 
The derivative of (\ref{kine1}) is given by:
\begin{equation}
    \begin{aligned}
    \boldsymbol{\dot{p}_x }
    & = \boldsymbol{\dot{p}_\eta} + R^i_b \boldsymbol{\dot{P}^b_{ee/b}} + \boldsymbol{\omega^i_{ee/b}} \times \boldsymbol{P^i_{ee/b}}, \\
        \boldsymbol{\omega^i_{ee/i}} &= \boldsymbol{\omega^i_{b/i}} + R^i_b \boldsymbol{\omega^b_{ee/b}},
    \end{aligned} 
\label{kine2}
\end{equation}
where $\boldsymbol{\omega^i_{b/b}}$ denotes the angular velocity of the end effector expressed in $\{b\}$, and $\times$ denotes the cross-product operator. 
The kinematic relation between $\dot\eta$ and $\dot x$ is then expressed as:
\begin{equation}
    \begin{aligned}
    \boldsymbol{\dot x} = \boldsymbol{\dot \eta} + \begin{bmatrix} {R^i_b} & \mathbf{0} \\ \mathbf{0} & {R^i_b} \end{bmatrix} \begin{bmatrix} \boldsymbol{\dot p^b_{ee/b}} \\ \boldsymbol{\omega^b_{ee/b}} \end{bmatrix} + \begin{bmatrix} \boldsymbol{\omega^i_{ee/b}} \times \boldsymbol{P^i_{ee/b}} \\ \boldsymbol{0} \end{bmatrix} .
    \end{aligned} 
\label{kine6}
\end{equation}

Consider the transformation of the manipulator between joint
space and Cartesian space: 
\begin{equation}
    \begin{aligned}
    \begin{bmatrix} \boldsymbol{\dot p^b_{ee/b}} \\ \boldsymbol{\omega^b_{ee/b}} \end{bmatrix}  = J(\boldsymbol{q}) \boldsymbol{\dot q}.
    \end{aligned} 
\label{jacobian}
\end{equation}
Substituting (\ref{jacobian}) into the kinematic equation (\ref{kine6}), the overall kinematic equation can be rewritten as:
\begin{equation}
    \begin{aligned}
    \boldsymbol{\dot x} = \boldsymbol{\dot \eta} + \hat{J} \boldsymbol{\dot q} + \boldsymbol{d},
    \end{aligned} 
\label{kine7}
\end{equation}
where $\boldsymbol{d} = \begin{bmatrix} \boldsymbol{\omega^i_{ee/b}} \times \boldsymbol{P^i_{ee/b}} \\ \boldsymbol{0} \end{bmatrix}$. The derivative of (\ref{kine7}) is further derived as:
\begin{equation}
    \begin{aligned}
    \boldsymbol{\Ddot x} = \boldsymbol{\Ddot \eta} + \dot{\hat{J}} \boldsymbol{\dot q} + \hat{J} \boldsymbol{\Ddot q} + \boldsymbol{\dot d}.
    \end{aligned} 
\label{kine8}
\end{equation}

To directly control the motion and force of the end effector, according to the kinematic relationship (\ref{kine7}) and the derivative (\ref{kine8}), multiplying the pseudo-inverse and transpose
of the Jacobian matrix $(\hat{J}^{\dagger})^T$ on both sides of (\ref{dynamics2}) yields the task space dynamic coupling-integrated model of the manipulator on the moving base:
\begin{equation}
    \begin{aligned}
    M (\boldsymbol{\Ddot{x}} - \boldsymbol{\Ddot{\eta}} -\boldsymbol{\dot d} ) + C (\boldsymbol{\dot{x}}-\boldsymbol{\dot{\eta}} -\boldsymbol{d}) + G =  \boldsymbol{f + f_{e} + f_d},
    \end{aligned} 
    \label{dynamics2.5}
\end{equation}
where $M = (\hat{J}^\dagger)^T M_q \hat{J}^\dagger$ 
, $C =  (\hat{J}^\dagger)^T C_q \hat{J}^\dagger - (\hat{J}^\dagger)^T M_q \hat{J}^\dagger \dot{\hat{J}} \hat{J}^\dagger$, 
$G = (\hat{J}^\dagger)^T G_q$, and $f = (\hat{J}^\dagger)^T \boldsymbol{\tau}$. 
The system matrices can be separated into known and unknown matrices $M=M_0 + \Delta M$, $C=C_0 + \Delta C$, and $G=G_0 + \Delta G$, where $M_{0}$, $C_{0}$, $G_{0}$ are the known system matrices, $\Delta M$, $\Delta C$, $\Delta G$ denote the system's uncertainties.
Assume that the uncertainties are bounded:
\begin{equation}
\begin{aligned}
||\Delta M|| &= ||M_{0} - M|| \leq \delta M, \\
||\Delta C|| &= ||C_{0} - C|| \leq \delta C, \\
||\Delta G|| &= ||G_{0} - G|| \leq \delta D, \\
||\boldsymbol{f_d} || & \leq \delta f_d,
\end{aligned}
\label{uncertainty}
\end{equation}
where $\delta M$, $\delta C$, $\delta D$, $\delta f_d$ are the bounds of uncertainties. The dynamic model is then given by:
\begin{equation}
    \begin{aligned}
    M_0 \boldsymbol{\Ddot{x}}  + C_0 \boldsymbol{\dot{x}} + G_0 =  \boldsymbol{f + f_{e} + \mu_{c} + \mu_u},
    \end{aligned} 
    \label{dynamics3}
\end{equation}
where $\boldsymbol{\mu_c}$ denotes the dynamic coupling of the mobile base and $\boldsymbol{\mu_u}$ denotes other unmodeled uncertainties, with:
\begin{equation}
    \begin{split}
        \boldsymbol{\mu_c} &= M_0(\boldsymbol{\Ddot{\eta}} +\boldsymbol{\dot d}) + C_0 (\boldsymbol{\dot{\eta}}+\boldsymbol{d}),\\
        \boldsymbol{\mu_u} &= \boldsymbol{f_d} - \Delta M (\boldsymbol{\Ddot{x}}-\boldsymbol{\Ddot{\eta}}-\boldsymbol{\dot d}) - \Delta C (\boldsymbol{\dot{x}}-\boldsymbol{\dot{\eta}}-\boldsymbol{d}) - \Delta G.
    \end{split} 
    \label{disturbance}
\end{equation}

The dynamic coupling-integrated model (\ref{kine7}) and (\ref{dynamics3}) provide a direct means to design a Cartesian controller for the manipulator on the moving base.


\section{Motion/Force Control Approach via Extended UDE}
\noindent The overall diagram of the proposed dynamic motion/force control approach is illustrated in Fig.~\ref{control flow}. The feedforward term $\boldsymbol{f_c}$ predicts the dynamic coupling terms, thereby improving the dynamic response of the REI performance, whereas the feedback term $\boldsymbol{f_{u}}$ maintains the system's stability against other unmodeled dynamics.

\subsection{Extended UDE Design}
To improve the disturbance rejection ability of the mobile manipulator system, an extended UDE is designed to compensate for dynamic coupling and unmodeled uncertainties between the mobile base and the manipulator.

The mobile manipulator system is designed to exhibit the impedance behavior at the end effector \citep{MTI21}. Specifically, the desired impedance behavior is expressed in the form of: 
\begin{equation}
    \begin{aligned}
    M_d (\boldsymbol{\Ddot{x}}-\boldsymbol{\Ddot{x}_d}) + C_d \boldsymbol{\dot{e}} + K_d \boldsymbol{e} = K_{f,d} \boldsymbol{e_f},
    \end{aligned} 
    \label{impedancemodel0}
    \end{equation}
where $M_d$, $C_d$, $K_d$, $K_{f,d}$ represent the desired inertial, damping, stiffness, and force matrices. Motion and force error terms are defined as $\boldsymbol{e} = \boldsymbol{x}-\boldsymbol{{x}_d}$, $\boldsymbol{e_f} = \boldsymbol{f_e} - \boldsymbol{f_{e,d}}$. Without loss of generality, choosing $M_d = M_0$ yields:
\begin{equation}
    \begin{aligned}
    M_0 (\boldsymbol{\Ddot{x}}-\boldsymbol{\Ddot{x}_d}) + C_d \boldsymbol{\dot{e}} + K_d \boldsymbol{e} = K_{f,d} \boldsymbol{e_f}.
    \end{aligned} 
    \label{impedancemodel}
\end{equation}

Assume that the system dynamics and the unmodeled disturbance are bounded by a cutoff frequency $\omega_c$. Moreover, assume that $G_{f,1}(s)$ and $G_{f,2}(s)$ are ideal low-pass filters that have unit gains and zero phase shift when $\omega \leq \omega_c$ and zero gain when $\omega > \omega_c$. Under these conditions, uncertainties $\boldsymbol{\mu_c}$ and $\boldsymbol{\mu_u}$ can be estimated through $G_{f,1}(s)$ and $G_{f,2}(s)$ \citep{UDE19}:
\begin{equation}
    \begin{aligned}
    \boldsymbol{\hat{\mu}_c} &= \mathscr{L}^{-1} \{G_{f1}(s)\} \ast \boldsymbol{\mu_c}, \\
    \boldsymbol{\hat{\mu}_u} &= \mathscr{L}^{-1} \{G_{f2}(s)\} \ast \boldsymbol{\mu_u},
    \end{aligned} 
    \label{f4}
\end{equation}
where $\ast $ represents the convolution symbol and $\mathscr{L}^{-1}$ represents the inverse Laplace transform symbol, $G_{fi} = diag\{G_{fi}[1],\cdots,G_{fi}[6]\}, i = 1,2$.
This indicates that UDE can effectively estimate unmodeled uncertainties in the system, thus improving the disturbance rejection ability of the mobile manipulator system.

In order to improve the transient response of the mobile manipulator under dynamic coupling, an extended UDE term is proposed to compensate for the dynamic coupling term $\hat{\mu}_c$ in the feedforward control loop, while the feedback linearization term, $G_0 - \boldsymbol{f_e} + C_0 \boldsymbol{\dot{x}_d}$, is also considered to compensate for the gravity and external wrenches exerted on the manipulator and ensures the tracking of the desired impedance behavior. The feedforward control input $f_c$ is formulated as:
\begin{equation}
    \begin{split}
    \boldsymbol{f_{c}} &= -\mathscr{L}^{-1} \{G_{f1}(s)\} \ast \boldsymbol{\mu_c} +C_0 \boldsymbol{\dot{x}_d} + G_0 - \boldsymbol{f_{e}} \\
    &= -\mathscr{L}^{-1} \{sG_{f1}(s)\} \ast M_0(\boldsymbol{\dot{\eta}} +\boldsymbol{d}) \\
    &- \mathscr{L}^{-1} \{G_{f1}(s)\} \ast C_0 (\boldsymbol{\dot{\eta}}+\boldsymbol{d}) \\
    &+C_0 \boldsymbol{\dot{x}_d} + G_0 - \boldsymbol{f_{e}},
    \end{split}
    \label{f_ff} 
\end{equation}
which combines the compensation terms for the dynamic coupling terms, gravity, and external wrenches. Substituting (\ref{f_ff}) into the dynamics (\ref{dynamics3}) results in: 
\begin{equation}
    \begin{aligned}
    M_0 \boldsymbol{\Ddot{x}}  + C_0 \boldsymbol{\dot{e}} = \boldsymbol{f_u} + \boldsymbol{\mu_u},
    \end{aligned} 
    \label{ude}
\end{equation}
which indicates that all the unmodeled uncertainties along with tracking error $C_0 \boldsymbol{\hat {\dot e}}$ can be estimated through system states and the system outputs as follows: 
\begin{equation}
    \begin{aligned}
    \boldsymbol{\hat \mu_u} - C_0 \boldsymbol{\hat {\dot e}} =  \mathscr{L}^{-1} \{G_{f2}(s)\} \ast (M_0 \boldsymbol{\Ddot{x}} - \boldsymbol{f_u}).
    \end{aligned} 
    \label{muuhat}
\end{equation}

Substituting the desired impedance model (\ref{impedancemodel}), (\ref{f4}), and (\ref{muuhat}) into (\ref{ude}), the feedback control input can be derived as: 
\begin{equation}
    \begin{split}
        \boldsymbol{f_u} &= M_0 \boldsymbol{\Ddot{x}_d} - (C_d \boldsymbol{\dot{e}} + K_d \boldsymbol{e} - K_f \boldsymbol{e_f}) \\ 
        &- \mathscr{L}^{-1} \{G_{f2}(s)\} \ast (M_0 \boldsymbol{\Ddot{x}} - \boldsymbol{f_u}).
    \end{split} 
    \label{f_ude_filter}
\end{equation}
By rewriting (\ref{f_ude_filter}), the final control input $\boldsymbol{f_{u}}$ is derived as:
\begin{equation}
    \begin{split}
    \boldsymbol{f_{u}} 
      & = \mathscr{L}^{-1} \{\frac{1}{1-G_{f2}(s)} \} \ast (M_0 \boldsymbol{\Ddot{x}_d } - C_d \boldsymbol{\dot{e}} - K_d \boldsymbol{e} + K_{f,d} \boldsymbol{e_f)} \\
      & - \mathscr{L}^{-1} \{\frac{G_{f2}(s)}{1-G_{f2}(s)} \} \ast M_0 \boldsymbol{\Ddot{x}} \\
      & = \mathscr{L}^{-1} \{\frac{1}{1-G_{f2}(s)} \} \ast (M_0 \boldsymbol{\Ddot{x}_d} 
          - C_d \boldsymbol{\dot{e}} - K_d \boldsymbol{e} + K_{f,d} \boldsymbol{e_f}) \\
      & - \mathscr{L}^{-1} \{\frac{sG_{f2}(s)}{1-G_{f2}(s)} \} \ast M_0 \boldsymbol{\dot{x}}.
    \end{split}
    \label{f_UDE} 
\end{equation}
It is worth noting that the control approach through extended UDE (\ref{f_ff}) and (\ref{f_UDE}) relies solely on velocity rather than acceleration information, making it easier to implement in practice. In the proposed control approach, the extended UDE is utilized as a compensator within the controller, allowing it to estimate the dynamic coupling term in the feedforward loop and other disturbances by filtering system inputs and states. 

\subsection{Stability Analysis}
\begin{theorem}
\normalfont Suppose that the system's uncertainty and external disturbance are bounded, as shown in (\ref{uncertainty}), and that the system dynamics and disturbances are constrained below the given frequency $\omega_c$. Under the control law given by (\ref{f_all}):
\begin{equation}
    \begin{aligned}
    \boldsymbol{f} = \boldsymbol{f_{c}} + \boldsymbol{f_{u}}, 
    \end{aligned}
    \label{f_all} 
\end{equation}
the manipulator system on the mobile base (\ref{dynamics3}) is globally asymptotically stable in the full motion control mode. Furthermore, it remains stable and achieves the desired hybrid impedance model (\ref{impedancemodel}) in motion/force control mode, provided that there exist 
$\boldsymbol{\delta_u} = diag\{\delta_{u1}, \delta_{u2}, \delta_{u3}, \delta_{u4}$, $\delta_{u5}, \delta_{u6}\} \succ {0}$, and $\hat C_d \in \mathbb{R}^{6\times6} \succ {0}$ such that the following conditions hold:
\begin{equation}
    \begin{aligned}
    |\boldsymbol{\dot{e}}^T \boldsymbol{\mu_u}| & \leq\boldsymbol{\dot{e}} ^T \boldsymbol{\delta_u} \boldsymbol{\dot{e}} ,\\
    \hat C_d &= C_d 
    - \boldsymbol{\delta_u} \succ {0}.
    \end{aligned}
    \label{const}
\end{equation}
\end{theorem}

\begin{proof}
Consider a Lyapunov candidate:
\begin{equation}
    \begin{aligned}
    V = \frac{1}{2}\boldsymbol{\dot{e}}^T M_0 \boldsymbol{\dot{e}} + \frac{1}{2}\boldsymbol{e}^T K_d \boldsymbol{e}.
    \end{aligned}
    \label{lyapunov} 
\end{equation}
The time derivative of the Lyapunov candidate $V$ is given by:
\begin{equation}
    \begin{split}
    \dot{V} & = \boldsymbol{\dot{e}}^T \left( M_0 \boldsymbol{\Ddot{e}} + \frac{1}{2} \dot{M}_0 \boldsymbol{\dot{e}} \right) + \boldsymbol{{e}}^T  K_d \boldsymbol{\dot{e}}. 
    \end{split}
    \label{lyapunov_dot} 
\end{equation}

Substituting the dynamic model of the manipulator on the moving base (\ref{dynamics3}) into (\ref{lyapunov_dot}) results in the following:
\begin{equation}
    \begin{split}
    \dot{V} &= 
            \boldsymbol{\dot{e}}^T ( 
            -C_0 \boldsymbol{\dot{x}} - G_0 + \boldsymbol{f_e} + \boldsymbol{f} + \boldsymbol{\mu_c} + \boldsymbol{\mu_u}\\ 
            & - M_0 \boldsymbol{\Ddot{x}_d} + \frac{1}{2} \dot{M}_0 \boldsymbol{\dot{e}} ) 
             +\boldsymbol{ {e}}^T  K_d \boldsymbol{\dot{e}}.
    \end{split}
    \label{lyapunov_dot2} 
\end{equation}

Substituting control law $f$ (\ref{f_ff}), (\ref{f_ude_filter}), and (\ref{f_all}) into (\ref{lyapunov_dot2}) produces:
\begin{equation}
    \begin{split}
    \dot{V} & = \boldsymbol{\dot{e}}^T ( - C_d \boldsymbol{\dot{e}} + K_{f,d} \boldsymbol{e_f} - \boldsymbol{\mu_{u}}). 
    \end{split}
    \label{lyapunov_dot4} 
\end{equation}

Substituting (\ref{const}) into (\ref{lyapunov_dot4}), the following relation can be obtained:
\begin{equation}
    \begin{split}
        \dot{V} & \leq - \boldsymbol{\dot{e}}^T C_d \boldsymbol{\dot{e}} + \boldsymbol{\dot{e}}^T K_{f,d} \boldsymbol{e_f} +|\boldsymbol{\dot{e}}^T \boldsymbol{\mu_{u}}| \\
        & \leq - \boldsymbol{\dot{e}}^T \hat{C}_d \boldsymbol{\dot{e}} + \boldsymbol{\dot{e}}^T K_{f,d} \boldsymbol{e_f} \leq \boldsymbol{\dot{e}}^T K_{f,d} \boldsymbol{e_f}.
    \label{passivity}
    \end{split}
\end{equation}

From (\ref{passivity}), it follows that when the mobile manipulator is in full motion control mode, $\dot{V} = -\boldsymbol{\dot{e}}^T \hat{C}_d \boldsymbol{\dot{e}}$, which is negative semi-definite. According to LaSalle's invariance principle, the system is globally asymptotically stable. When the system is in motion/force control mode, the dissipation equality can be derived from (\ref{passivity}) \citep{dissipative}:
\begin{equation}
    \begin{split}
        V(t) \leq V(0) + \int_0^t(\boldsymbol{\dot{e}}^T K_{f,d} \boldsymbol{e_f})dt,
    \label{dissipation}
    \end{split}
\end{equation}
which indicates that the closed-loop system (\ref{dynamics3}) (\ref{impedancemodel}) (\ref{f_all}) is dissipative and ensures a stable behavior. 
This completes the proof.
\end{proof}

\subsection{Dynamic Performance Analysis}
To evaluate the transient performance of the proposed extended UDE, this section compares the traditional UDE with extended UDE.

The traditional UDE, generally described as (\ref{f_ff}), considers the feedforward terms $C_0 \boldsymbol{\dot{x}_d} + G_0 - \boldsymbol{f_{e}}$. Then, considering the manipulator dynamics (\ref{dynamics3}), the dynamic coupling term $\mu_c$ is also compensated in the feedback loop as follows: 
\begin{equation}
    \begin{aligned}
    \boldsymbol{\hat \mu_c + \hat \mu_u + C_0 \dot e} =  \mathscr{L}^{-1} \{G_{f}(s)\} \ast (M_0 \boldsymbol{\Ddot{x}} - \boldsymbol{f_{UDE}}).
    \end{aligned} 
    \label{mu_ude_traditional}
\end{equation}

To follow the desired impedance behavior (\ref{impedancemodel}), the overall control law can be expressed as:
\begin{equation}
    \begin{split}
        \boldsymbol{f_{UDE}} &= M_0 \boldsymbol{\Ddot{x}_d} - (C_d \boldsymbol{\dot{e}} + K_d \boldsymbol{e} - K_f \boldsymbol{e_f}) \\
        &- \mathscr{L}^{-1} \{G_{f}(s)\} \ast (M_0 \boldsymbol{\Ddot{x}} - \boldsymbol{f_{UDE}})\\
        &+ C_0 \boldsymbol{\dot{x}_d} + G_0 - \boldsymbol{f_{e}}. 
    \end{split} 
    \label{f_ude_traditional}
\end{equation}
In real applications, the controller is discrete-time with a sampling period $T_s$. At $t= kT_s$, the discrete-time UDE controller can be represented by: 
\begin{equation}
    \begin{split}
        \boldsymbol{f_{UDE}(k)} &= M_0 \boldsymbol{\Ddot{x}_d(k)} - (C_d \boldsymbol{\dot{e}(k)} + K_d \boldsymbol{e}(k) - K_f \boldsymbol{e_f}(k)) \\
        &- \mathscr{L}^{-1} \{G_{f}(s)\} \ast (M_0 \boldsymbol{\Ddot{x}(k-1)} - \boldsymbol{f_{UDE}(k-1)})\\
        &+ C_0 \boldsymbol{\dot{x}_d(k)} + G_0 - \boldsymbol{f_{e}(k)}. 
    \end{split} 
    \label{f_ude_discrete_traditional}
\end{equation}
The system uncertainties (\ref{mu_ude_traditional}) at $t=kT_s$ are estimated based on the inputs and states of the system at $t=(k-1)T_s$. Due to time-step delays, UDE gradually reduces system errors, which results in a slower dynamic response. 

In contrast, the control input for the extended UDE (\ref{f4}) is given by:
\begin{equation}
    \begin{split}
        \boldsymbol{f_{UDE}(k)} &= M_0 \boldsymbol{\Ddot{x}_d(k)} - (C_d \boldsymbol{\dot{e}(k)} + K_d \boldsymbol{e}(k) - K_f \boldsymbol{e_f}(k)) \\
        &- \mathscr{L}^{-1} \{G_{f}(s)\} \ast (M_0 \boldsymbol{\Ddot{x}(k-1)} - \boldsymbol{f_{UDE}(k-1)})\\
        &-\mathscr{L}^{-1} \{sG_{f1}(s)\} \ast M_0(\boldsymbol{\dot{\eta}(k)} +\boldsymbol{d(k)}) \\
    &- \mathscr{L}^{-1} \{G_{f1}(s)\} \ast C_0 (\boldsymbol{\dot{\eta}(k)}+\boldsymbol{d(k)}) \\
        &+ C_0 \boldsymbol{\dot{x}_d(k)} + G_0 - \boldsymbol{f_{e}(k)},
    \end{split} 
    \label{f_extendude_discrete_traditional}
\end{equation}
where dynamic coupling term $\boldsymbol{\mu_b}$ is explicitly compensated 
in the feedforward loop to improve the speed of response of the system, while the feedback loop compensates for other unmodeled uncertainties. 

\section{Simulation and Experiment}
\subsection{Simulation Configuration}
To evaluate the effectiveness and dynamic performance of the proposed dynamic coupling model and extended UDE, a mobile manipulator platform is built in the Gazebo simulation environment. The system consists of a UR5e robotic arm mounted on a MiR-100 mobile base. All controllers operate at $125$ Hz in the Robot Operating System (ROS). The implementation is publicly available\footnote{https://https://github.com/SwonGao/Mobile-Manipulator-Simuator.}.  

Simulation studies are conducted to validate the dynamic coupling model and evaluate the performance of the motion/force control approach via extended UDE under a straight-line low-dynamic motion case and a high-dynamic motion case, respectively. 

Four control approaches are selected for ablation studies, namely C1, C2, C3, and C4. 
\begin{itemize}
    \item C1: the proposed control approach, impedance control (IC) incorporated with extended UDE and the dynamic-coupling integrated model;
    \item C2: IC incorporated with UDE and the dynamic-coupling integrated model;
    \item C3: IC incorporated with UDE;
    \item C4: IC.
\end{itemize}
For controller C1, $G_{f1}$ are selected as second-order low-pass filters, and $G_{f2}$ are selected as first-order low-pass filters, where $G_{f1}[i] = \frac{108s}{s^2+8.485s+36}, i = 1,\cdots,6$, $G_{f2}[i] = \frac{\omega_{ci}}{s+\omega_{ci}}$, and $\omega_c$ is selected as $[6,6,6,3,3,3]^T$.
The impedance parameters are initialized as $ K_d = diag\{200, 200, 200, 20$, $20$, $20\}$, $D_d = diag\{2, 2, 2, 1, 1, 1\}$, $K_{f,d} = diag\{0, 5, 0, 0, 0, 0\}$. For controllers C2, C3, and C4, the control parameters remain the same as C1. 

\subsection{Validation of Dynamic Coupling Terms}
\begin{figure} 
    \centering
  \subfloat[\label{veri1}]{%
       \includegraphics[width=\linewidth]{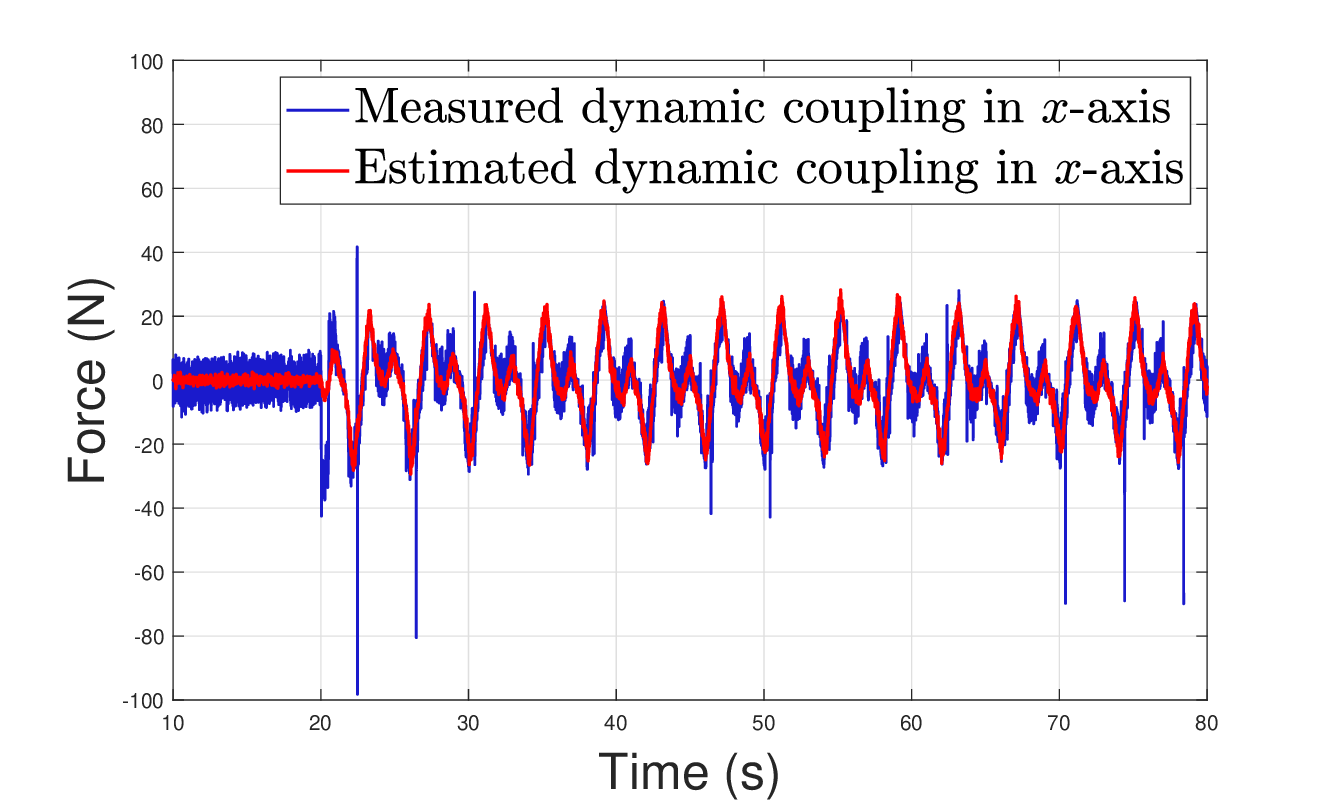}}
    \\ 
  \subfloat[\label{veri2}]{%
        \includegraphics[width=\linewidth]{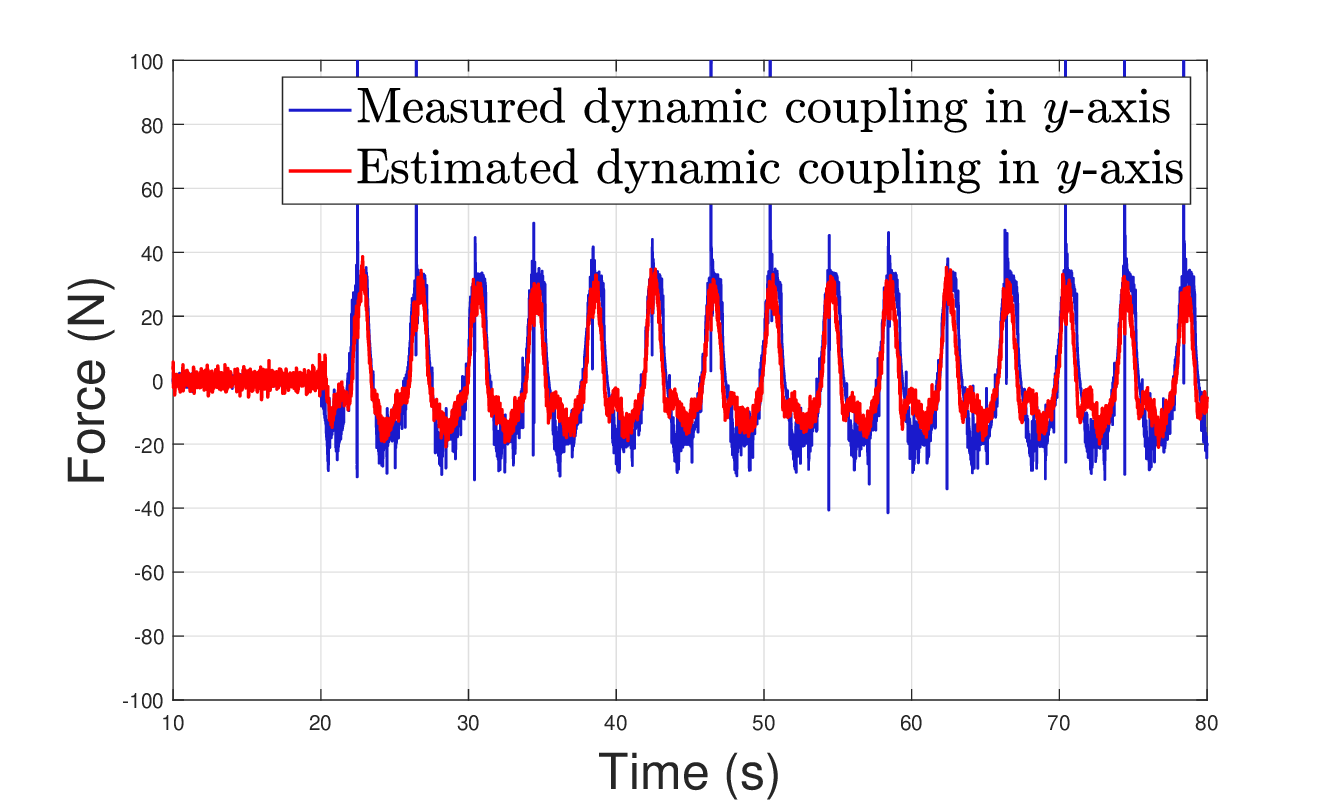}}
  \caption{Comparison of measured dynamic coupling terms in $x$-axis and $y$-axis and estimation using the dynamic coupling model.}
  \label{verification} 
\end{figure}
To validate the dynamic coupling model, the mobile base is assigned to move a high-dynamic motion trajectory, and the force sensor is equipped between the manipulator and the mobile base to measure the dynamic coupling terms. It is worth noticing that the force sensor also measures the gravity force and torque between the mobile base and the manipulator, which is compensated through the intrinsic stability of the mobile base. To better distinguish the dynamic coupling terms against gravity/torque from the manipulator and the mobile base, during the high-dynamic motion of the mobile base, the manipulator applies joint impedance control to maintain the joint position of $[1, 0.14, -1.45, 4.41, -1.35, 0 ]^T$.

The measured dynamic coupling terms are compared to predicted wrenches, i.e., the dynamic coupling terms from the proposed dynamic model. The comparison results between measurements and predictions are shown in Fig. \ref{verification}. Discrepancies are quantified using root mean squared error (RMSE), and mean absolute error (MAE) in the $x$ and $y$ directions, which are $2.71, 3.38$, and $6.77, 8.19$, for $x$ and $y$ directions respectively. The findings substantiate the efficiency of the proposed model.

\subsection{Dynamic Interaction Results}
Ablation studies are conducted to verify the proposed motion/force control approach's ability to handle dynamic coupling terms and other unmodeled uncertainties. The mobile manipulator is designed to perform a wall-wiping task. The wall is at $y=0.8m$, extending parallel along the x-axis.
In the first setting, under low-dynamic motion conditions, the mobile base moves forward along the wall at a speed of $0.2$ m/s, which causes random changes in the movement of the mobile base and leads to undesired motions of the manipulator. Meanwhile, the manipulator is tasked to maintain continuous contact with the wall while following a sinusoidal trajectory $x_{3,d}(t) = 0.8 + 0.1 sin(0.125 \pi (t-20))m$ along the z-axis and applying forces of $5 N$ and $10 N$ on the wall, respectively.

Furthermore, in the second setting, to further examine the impact of significant motion changes of the mobile base on the proposed control approach, the mobile base follows a high-dynamic sinusoidal trajectory on the y-axis, which simulates continuous dynamic couplings. The actual motion of the mobile base is illustrated in Fig. \ref{sim2_car}.
\begin{figure}[!t]
        \centering
        \subfloat[\label{aa}]{\includegraphics[width=\linewidth]{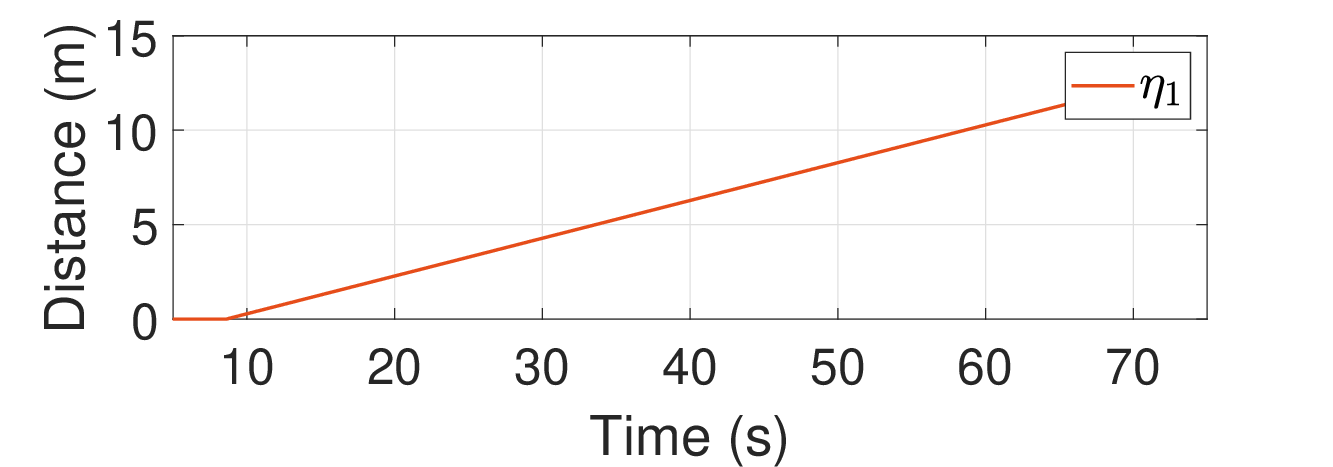}}\\
        \subfloat[\label{bb}]{\includegraphics[width=\linewidth]{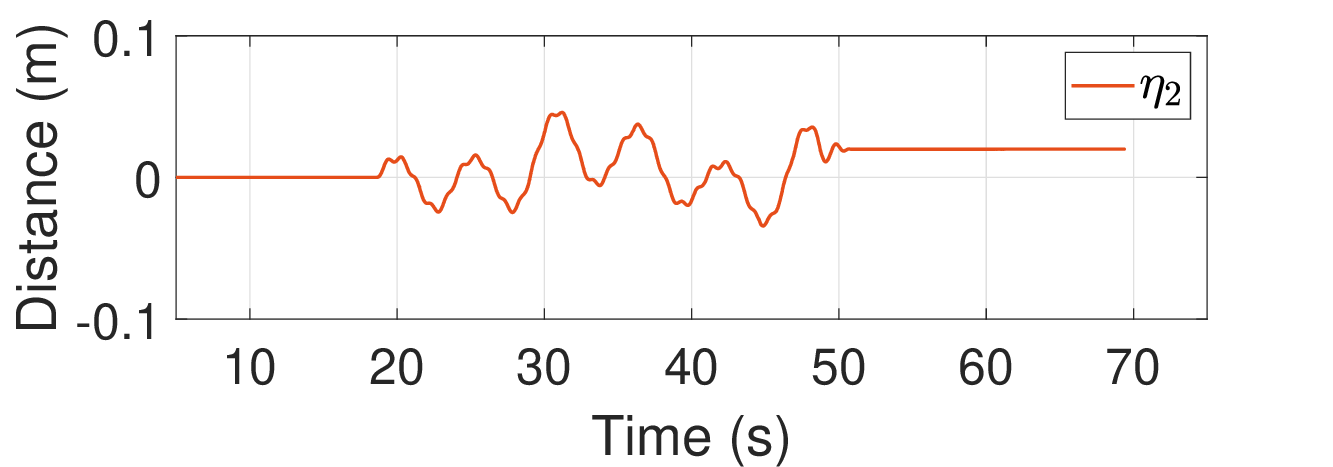}}\\
        \subfloat[\label{cc}]{\includegraphics[width=\linewidth]{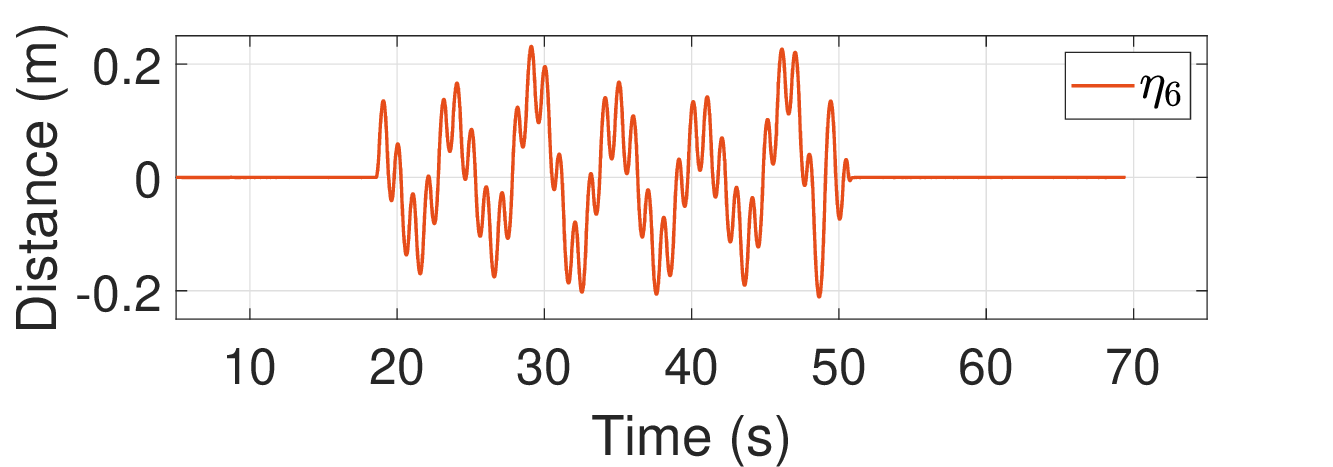}}
	\caption{The high-dynamic forward $\eta_1$, lateral $\eta_2$, and yaw $\eta_6$ motion trajectory of the mobile base. }
	\label{sim2_car}
 \end{figure}
 
\begin{figure}
    \centering
    \subfloat[\label{a}]{\includegraphics[width=\linewidth]{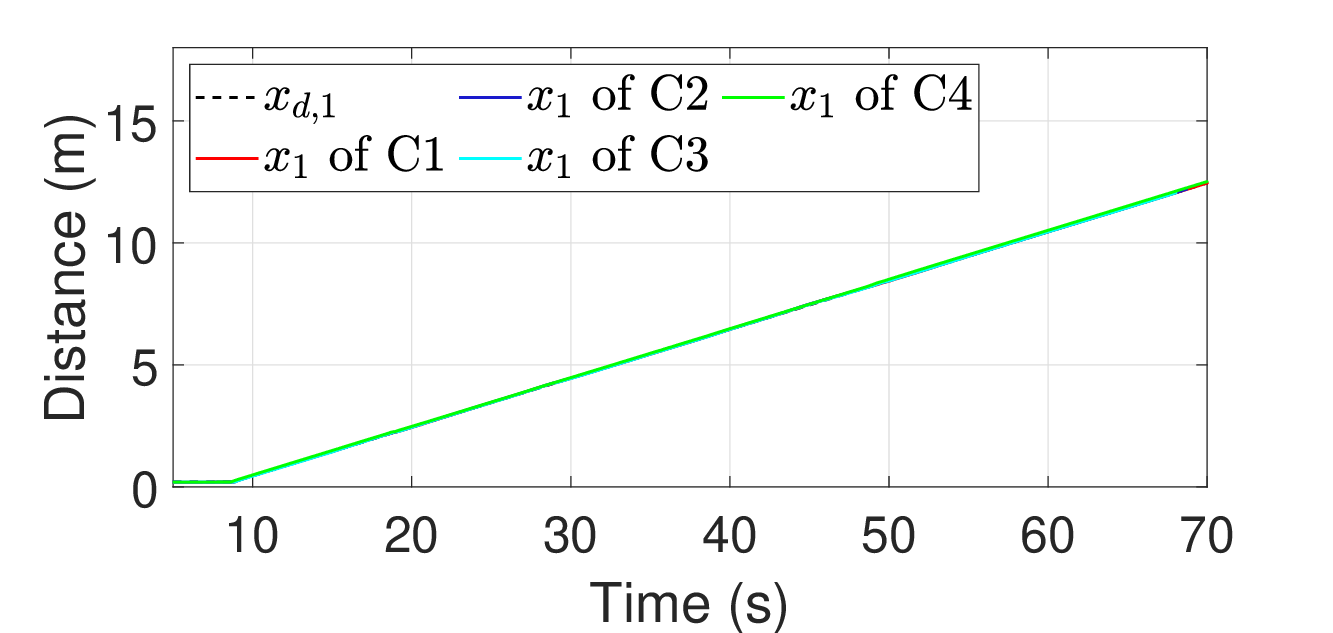}}\\
    \subfloat[\label{b}]{\includegraphics[width=\linewidth]{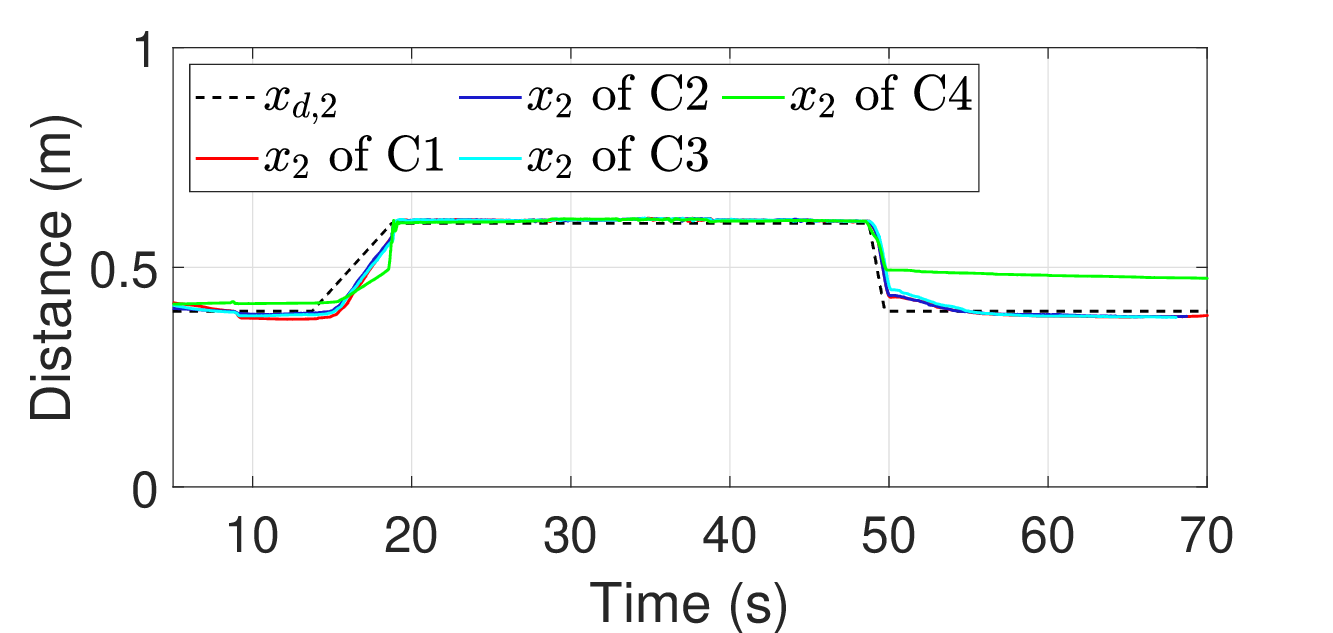}}\\
    \subfloat[\label{c}]{\includegraphics[width=\linewidth]{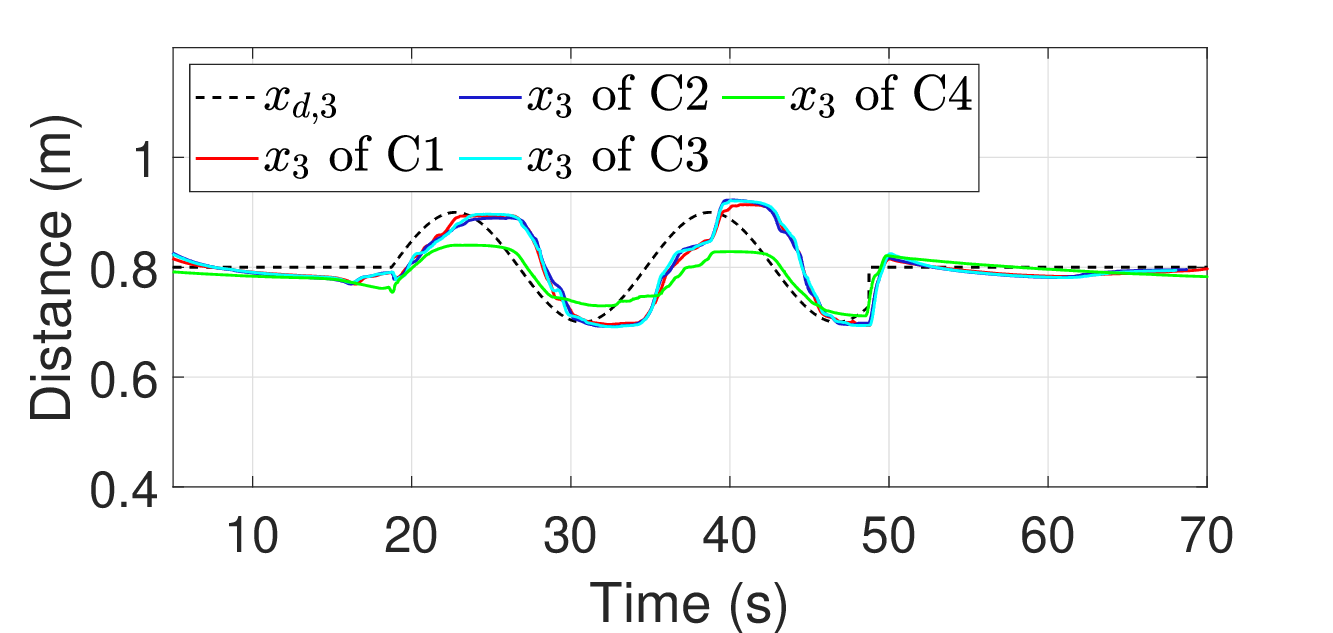}}
    \caption{Comparative motion tracking results with the mobile base in low-dynamic motion.}
    \label{sim1_motion}
\end{figure}

\begin{figure}
    \centering
    \includegraphics[width=\linewidth]{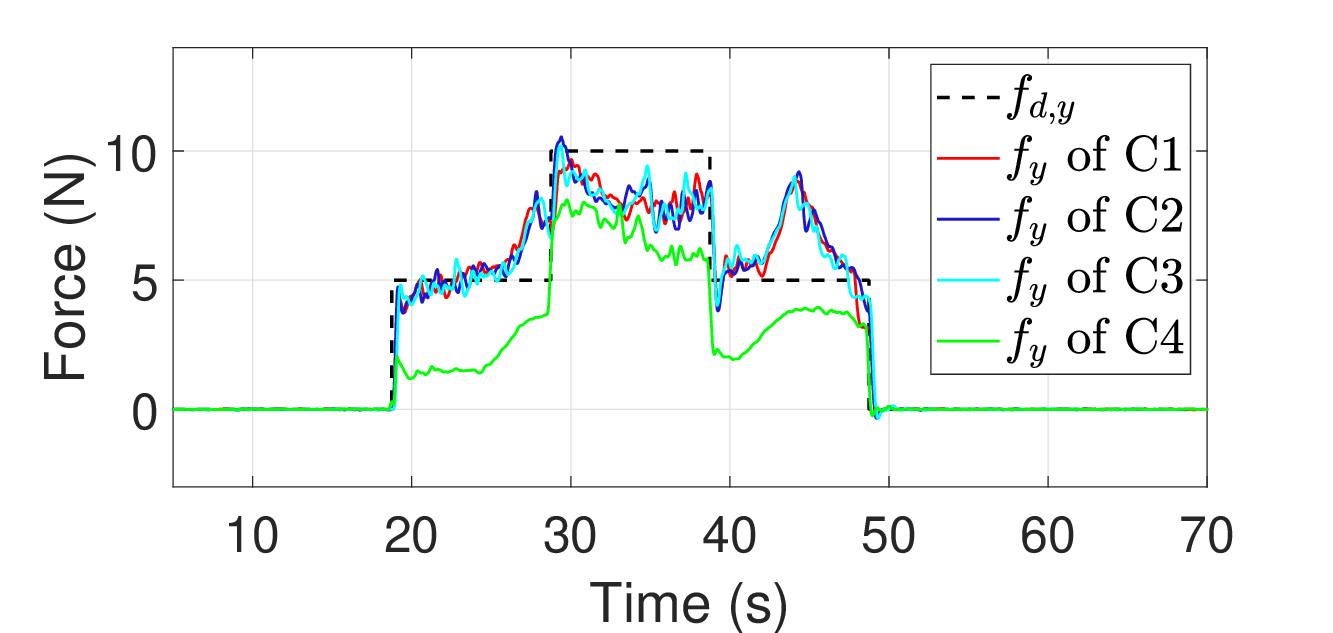}
    \caption{Comparative force tracking results with the mobile base in low-dynamic motion.}
    \label{sim1_force}
\end{figure}

\begin{table*}[!ht]
\renewcommand\arraystretch{1.25}
\centering
\caption{Average force tracking performance of four control schemes in dynamic interaction simulations. Each simulation is conducted 10 times to achieve more accurate results. The contents in brackets indicate the percentage improvement compared to the metrics of C4 (the first row).}
\label{sim_table}
\scalebox{0.9}{
    \begin{tabular}{ccccccc}
    \toprule
    & \multicolumn{3}{c}{{\underline{\hbox to 6mm{}Simulation with low-dynamic motion\hbox to 6mm{}}} } & \multicolumn{3}{c}{{\underline{\hbox to 6mm{}Simulation with high-dynamic motion\hbox to 6mm{}}} }             \\
    Controller  & {RMSE}    & {MAE} & {SSE} & {RMSE} & {MAE} & {SSE} \\
    \midrule
    {C4} & 2.76 & 2.50 & 2.50 & 2.80 & 2.30 & 2.13 \\
    {C3} & 1.70 (38.4\%) & 1.34 (46.4\%) & 0.05 (98.0\%) & 2.08 (25.7\%) & 1.67 (27.4) & 0.24 (88.7\%) \\
    {C2} & \textbf{1.67 {(39.4\%)}} & \textbf{1.30 {(48.0\%)}} & 0.06 (97.6\%) & 2.00 (28.6\%) & 1.60 (30.4\%)& 0.23 (89.2\%) \\ 
    {C1} & 1.69 (38.8\%) & 1.33 (46.8\%) & \textbf{0.05 ({98.0\%})} & \textbf{1.77 {(36.8\%)}} & \textbf{1.42 {(38.3\%)}} & \textbf{0.15 {(93.0\%)}} \\
    \bottomrule
    \end{tabular}
}
\end{table*}

The dynamic interaction results with the mobile base in low-dynamic motion are shown in Fig.~\ref{sim1_motion} and \ref{sim1_force}. As shown in Fig.~\ref{sim1_motion}, when the end effector switches back to full motion control mode ($t> 55s$), motions of C1, C2, and C3 converge to the setpoint while motions of C4 deviate from the setpoint due to the static friction of the joints, which indicates the effectiveness of the UDE in eliminating steady-state errors (SSEs).

The quantitative evaluation of force tracking performance of four different controllers in simulation with the low-dynamic motion is shown on the left side of Table \ref{sim_table}. Force tracking performance is measured by three metrics: RMSE, MAE, and SSE. 
Controllers with UDE C1, C2, and C3 achieve better force tracking results than C4. Furthermore, since the mobile base does not operate large high-motion, the performance of C1, C2, and C3 remains consistent.

\begin{figure}
    \centering
    \subfloat[\label{aaa}]{\includegraphics[width=\linewidth]{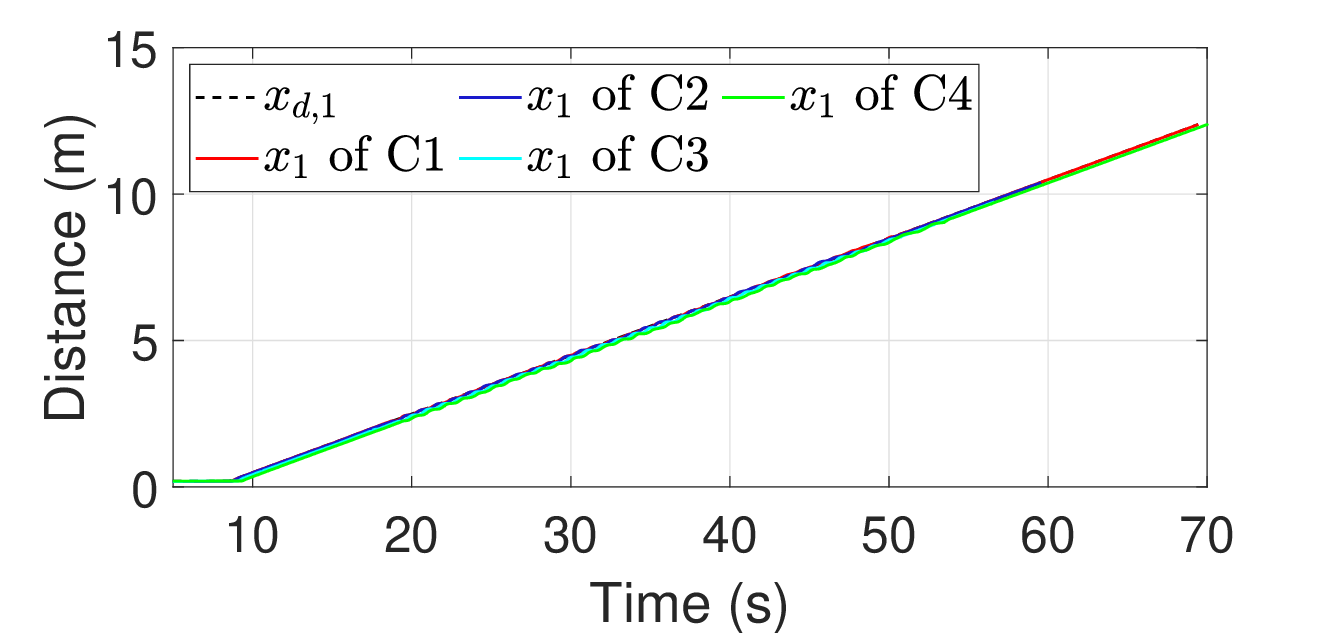}}\\
    \subfloat[\label{bbb}]{\includegraphics[width=\linewidth]{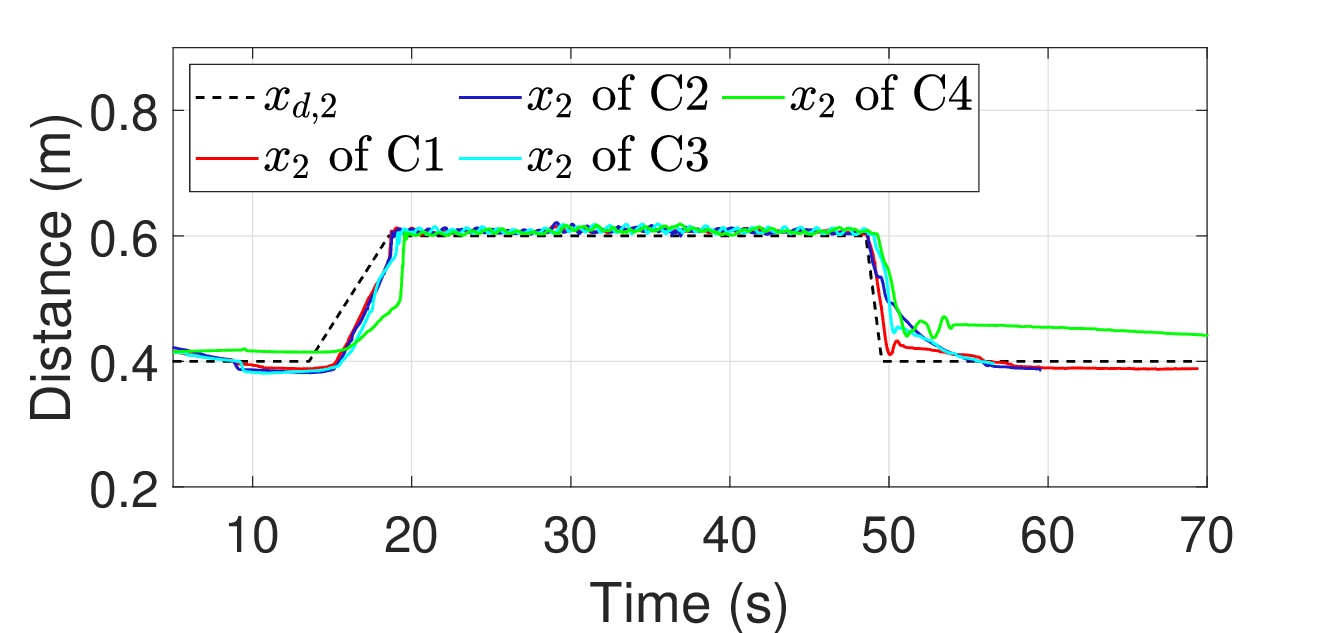}}\\
    \subfloat[\label{ccc}]{\includegraphics[width=\linewidth]{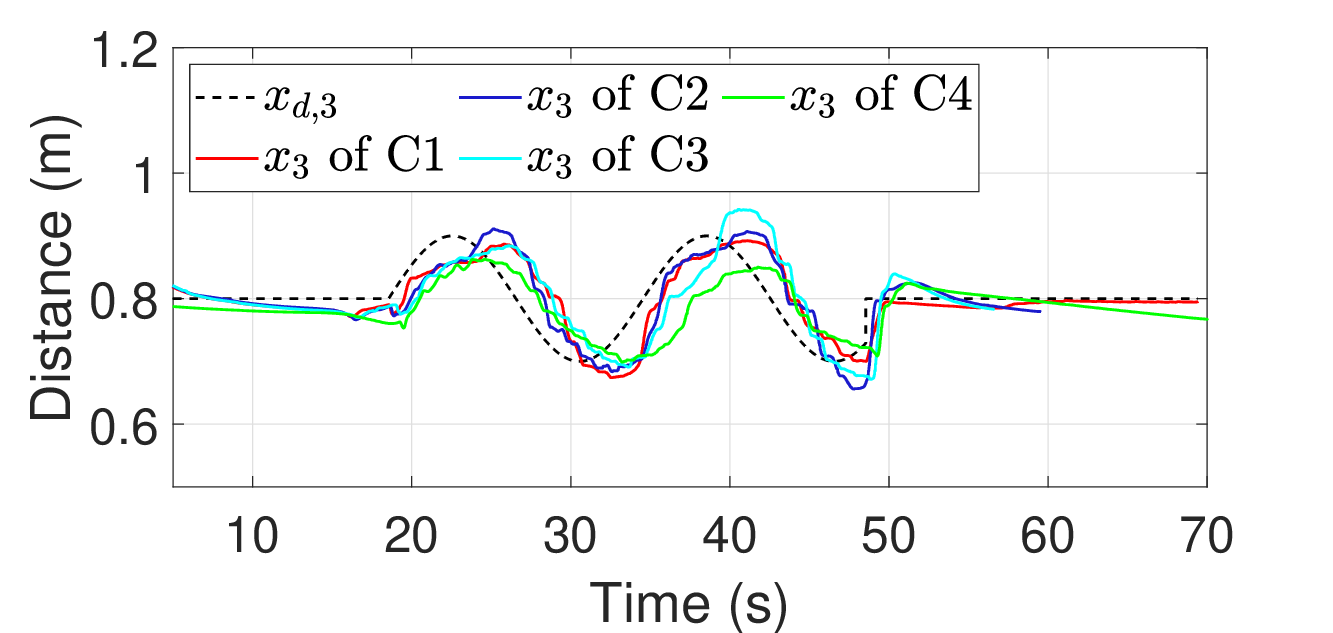}}
    \caption{Comparative motion tracking results with the mobile base in high-dynamic motion.}
    \label{sim2_motion}
\end{figure}

\begin{figure}
    \centering
    \includegraphics[width=\linewidth]{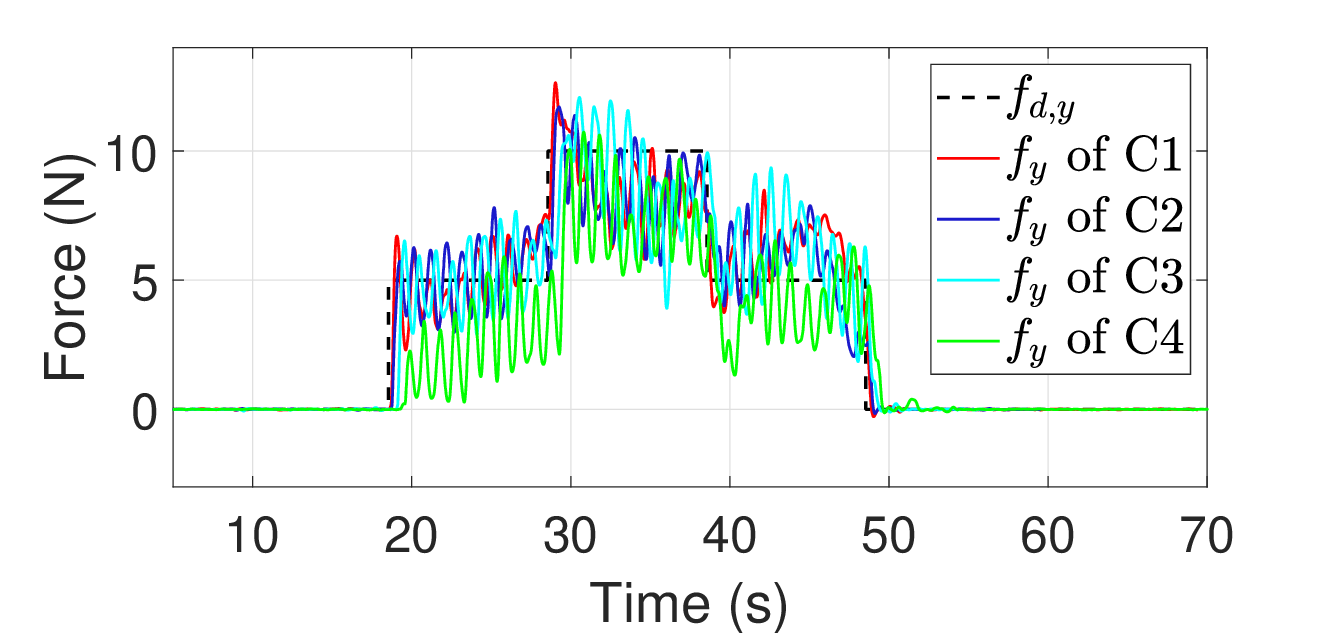}
    \caption{Comparative force tracking results with the mobile base in high-dynamic motion.}
    \label{sim2_force}
\end{figure}

The motion and force tracking results with the mobile base in high-dynamic motion are illustrated in Fig.~\ref{sim2_motion} and Fig.~\ref{sim2_force}. While significant dynamic movements of the mobile base have a negligible impact on the motion tracking performance of the end effector, the precision of force interaction is severely affected by dynamic disturbances. 

Quantitative results are presented on the right side of Table~\ref{sim_table}, which reveals that the performance of all controllers deteriorates due to the high-dynamic motion of the mobile base. Nevertheless, compared to other controllers, C1 demonstrates the superior force tracking performance of the end effector, with a clear hierarchy of performance: $C1 > C2 > C3>C4$. This indicates that the extended UDE, the proposed model, and the UDE all contribute positively to mitigating the dynamic coupling effects induced by the mobile base.
In addition, the comparison between simulation results with the mobile base in low-dynamic and high-dynamic motion demonstrates the deterioration in control performance due to dynamic coupling effects, as shown in Fig.~\ref{sim_decay}. The figure indicates that among all controllers, C1 experiences the least performance degradation, followed by C2 and C3, which demonstrates the effectiveness of the proposed approach. 

\begin{figure} 
    \centering
    \includegraphics[width=\linewidth]{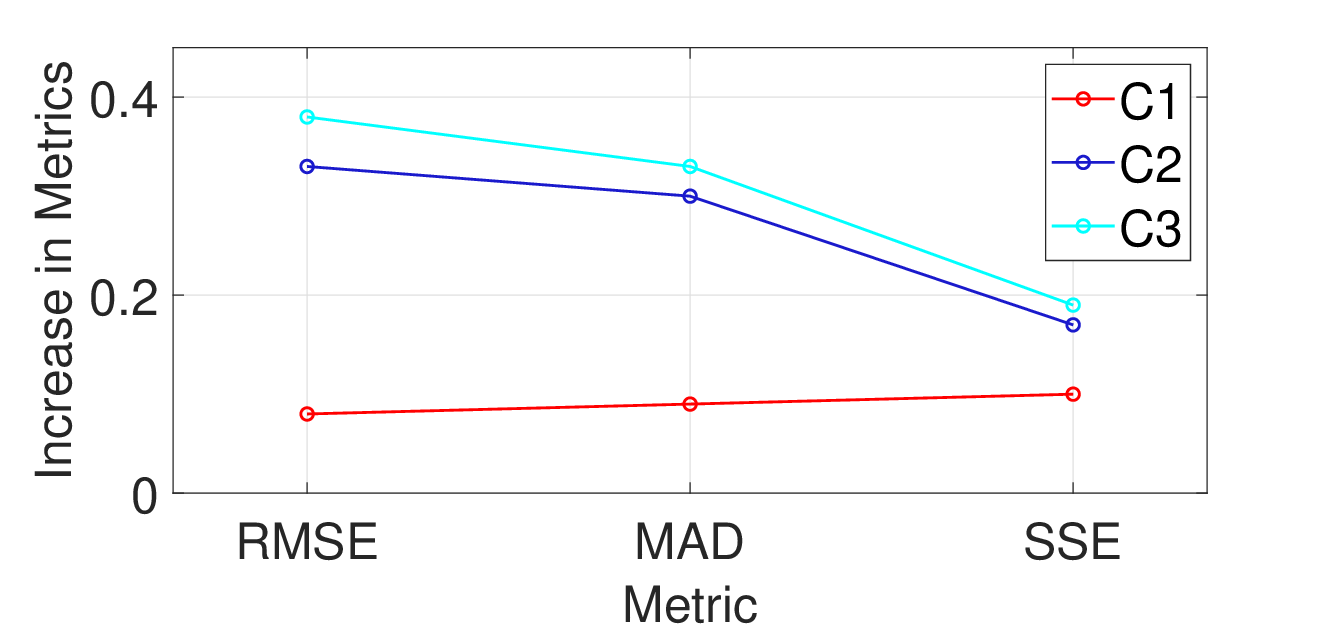}
    \caption{Performance degradations of different control approaches from low-dynamic to high-dynamic motion of the mobile base.}
  \label{sim_decay} 
\end{figure}

\subsection{Experiment Setup}
\begin{figure}[]
	\centering
	\includegraphics[scale=0.39]{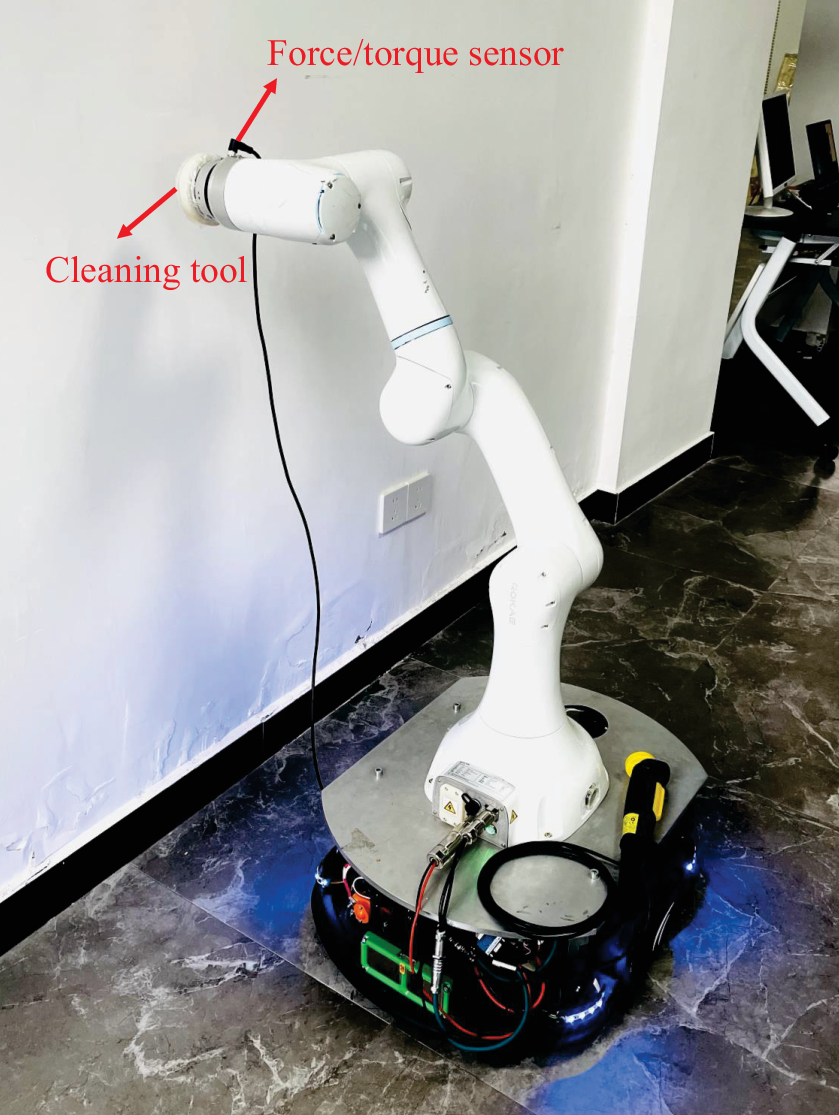}
	\caption{The prototype of the wall-cleaning mobile manipulator. The end effector is equipped with a 6 DOF force/torque sensor and a cleaning tool, and the mobile manipulator is tasked to follow a motion/force trajectory along the rigid wall.}
	\label{Experiment}
\end{figure}

\begin{figure*} 
    \centering
  \subfloat[\label{1a}]{%
       \includegraphics[width=0.5\linewidth]{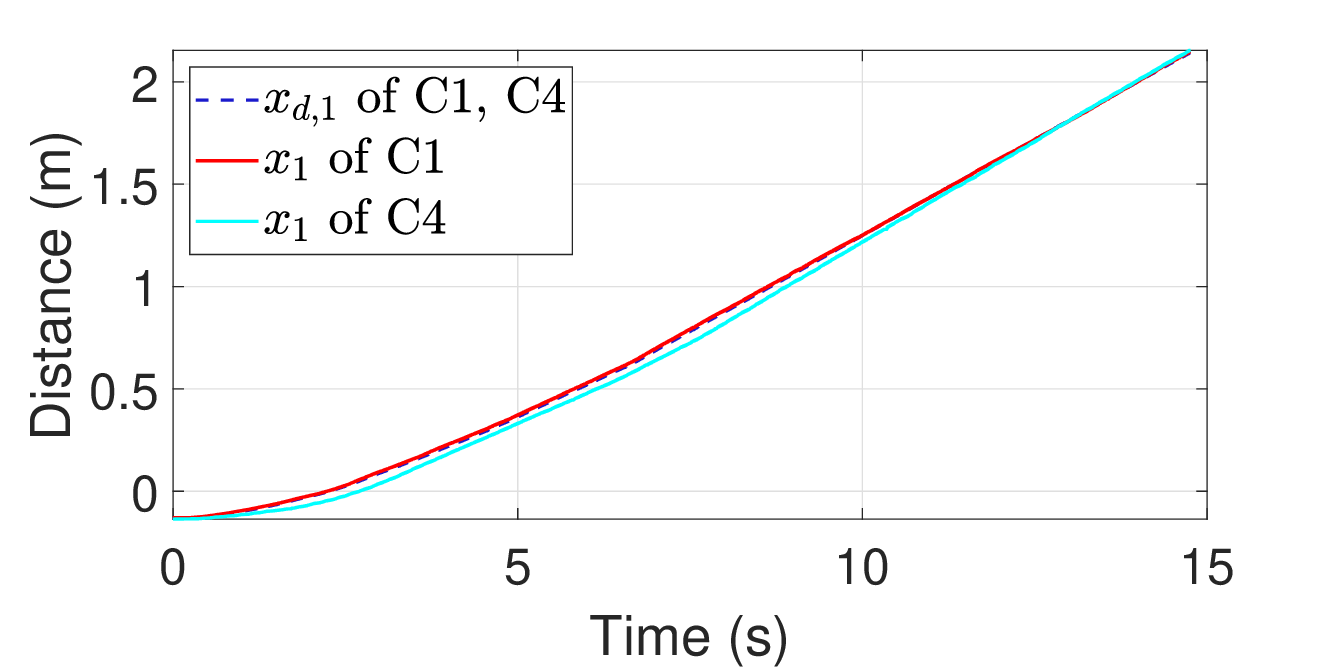}}
    \hfill
  \subfloat[\label{1b}]{%
        \includegraphics[width=0.5\linewidth]{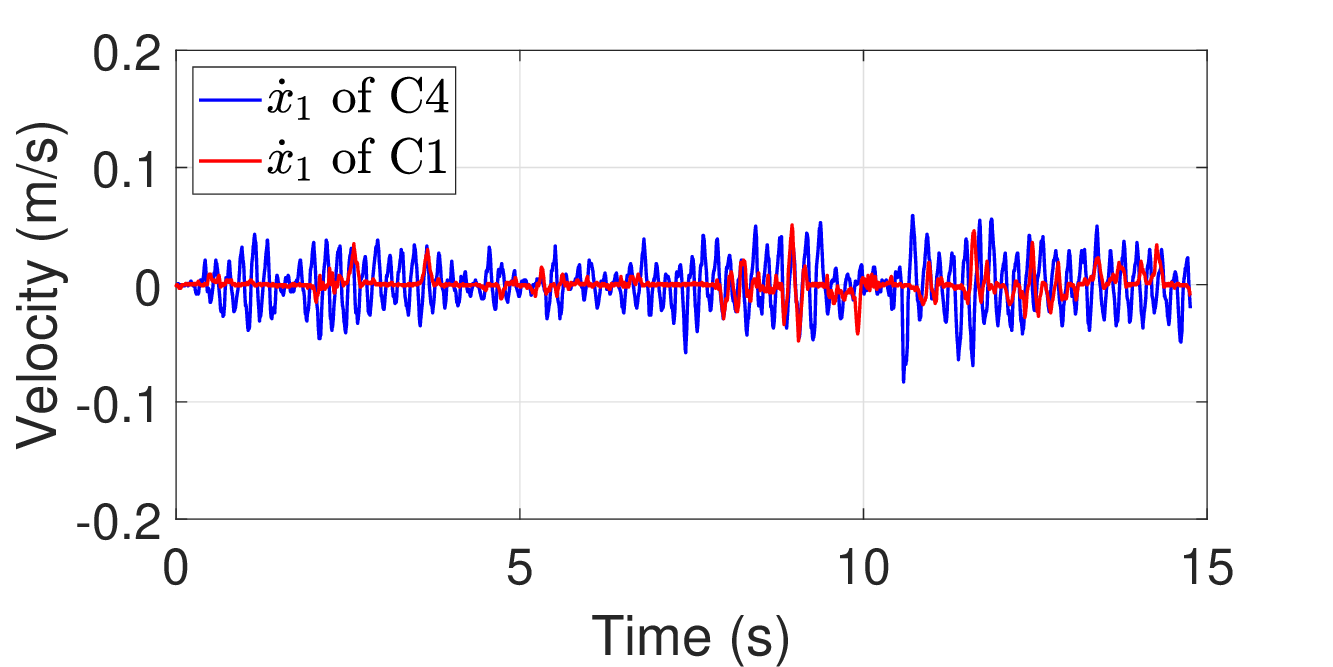}}
    \\
  \subfloat[\label{1c}]{%
        \includegraphics[width=0.5\linewidth]{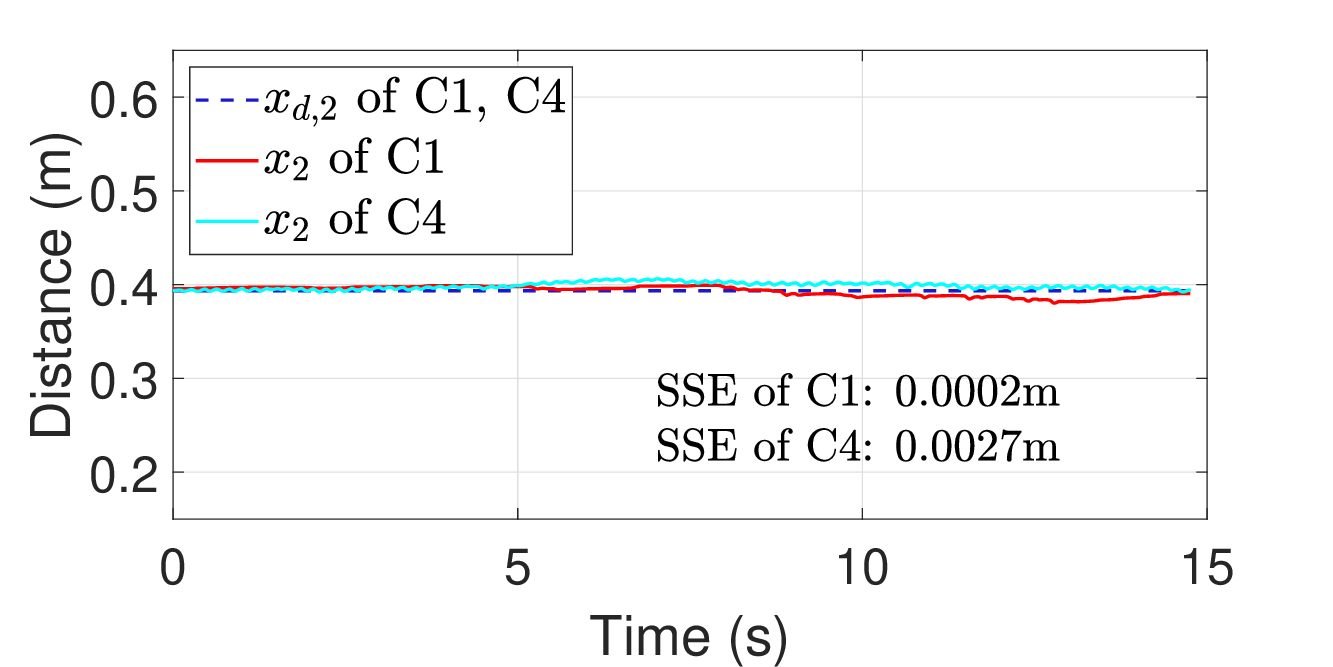}}
    \hfill
  \subfloat[\label{1e}]{%
        \includegraphics[width=0.5\linewidth]{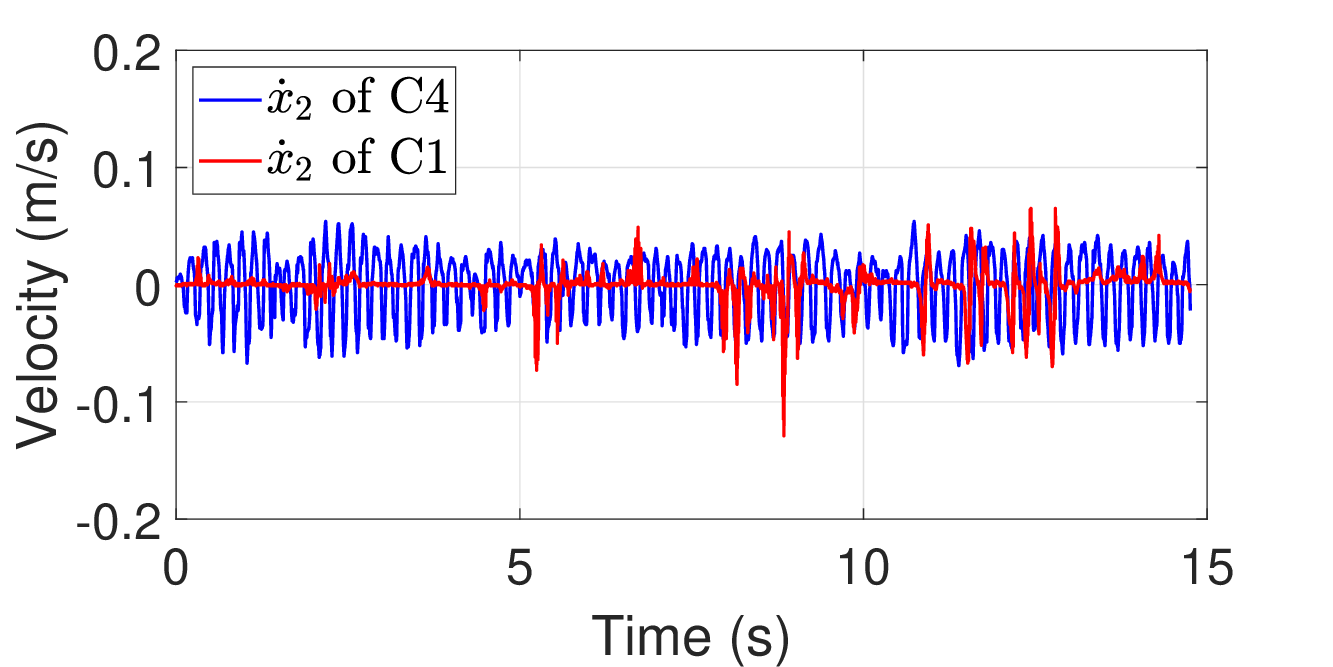}}
    \\
  \subfloat[\label{1f}]{%
        \includegraphics[width=0.5\linewidth]{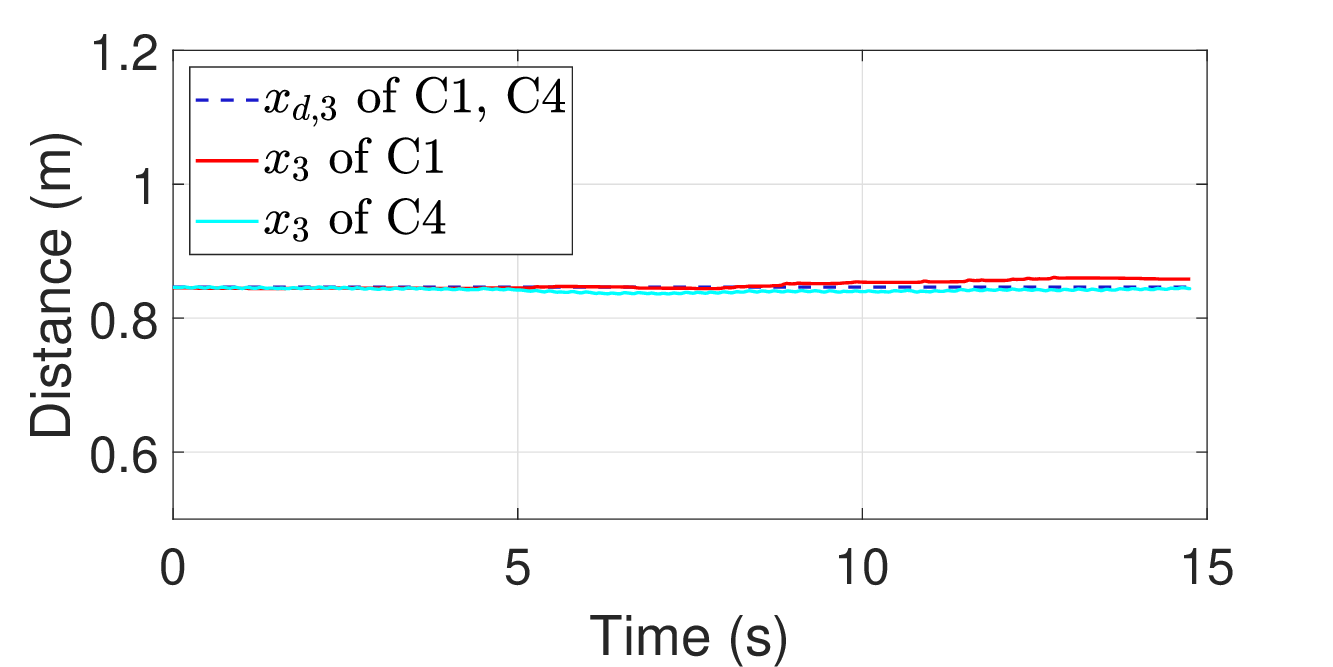}}
    \hfill
  \subfloat[\label{1g}]{%
        \includegraphics[width=0.5\linewidth]{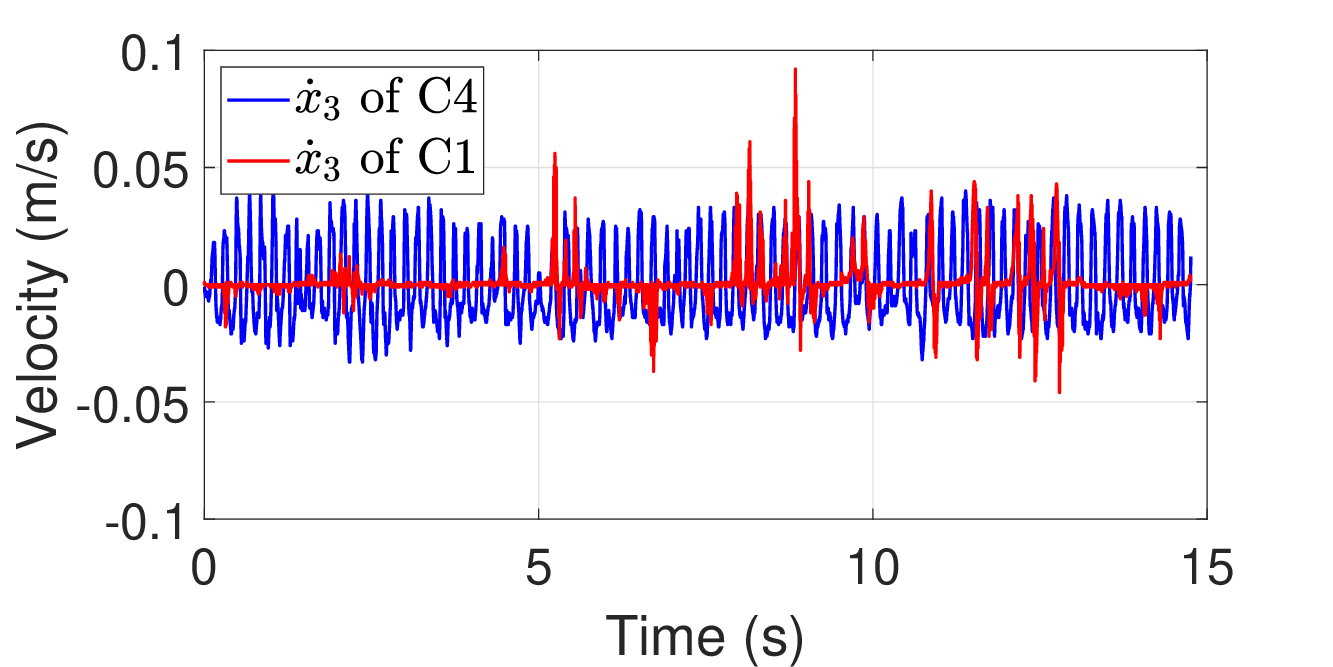}}
    \\
  \caption{Comparative motion tracking results of the experiment. (a), (c), (e) are the comparative results of position tracking between C1 and C4, while (b), (d), (f) are the comparative results of the velocity between C1 and C4.}
  \label{exp_pose} 
\end{figure*}

\begin{figure}[]
\centering
\includegraphics[width=\linewidth]{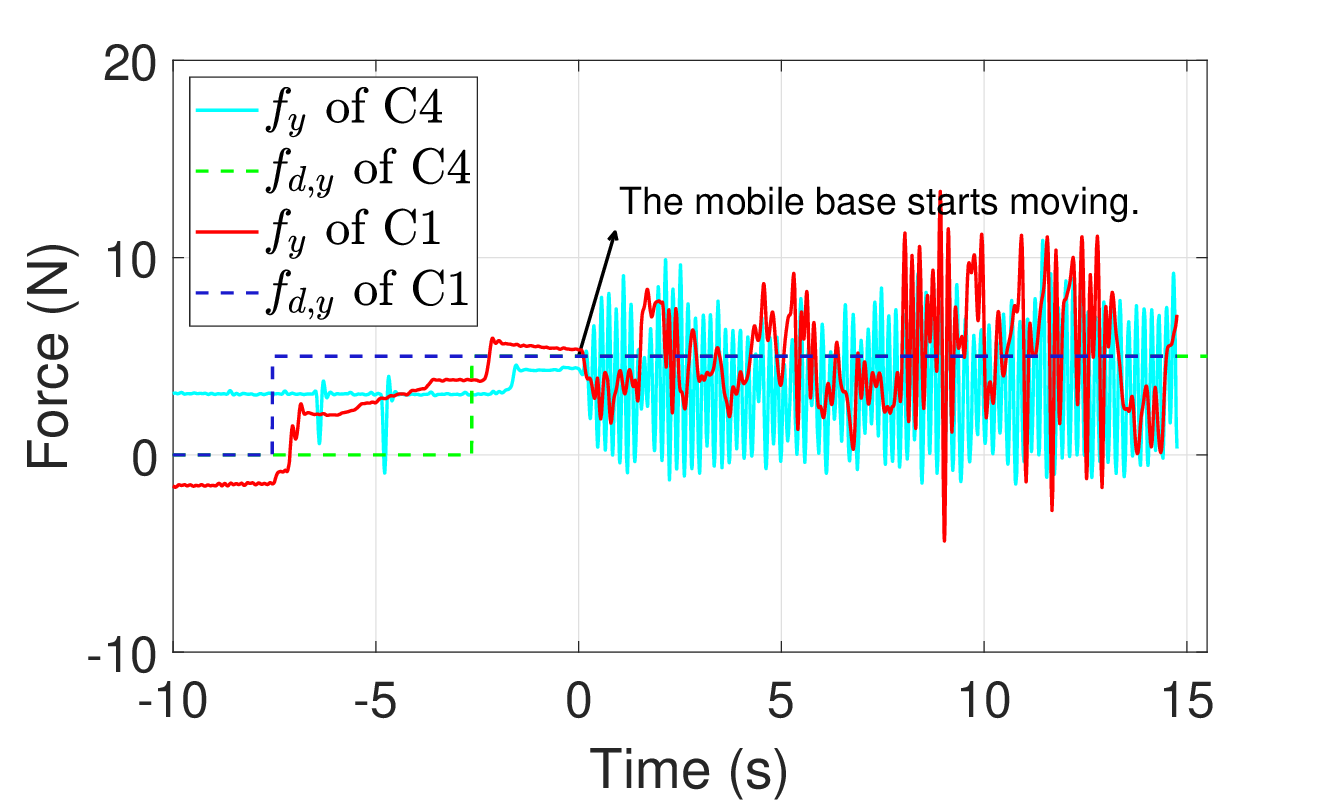}
\caption{Comparative results of force tracking between C1 and C4.}
\label{exp2_force_result}
\end{figure}

\begin{figure}[!t]
	\centering
	\includegraphics[width=\linewidth]{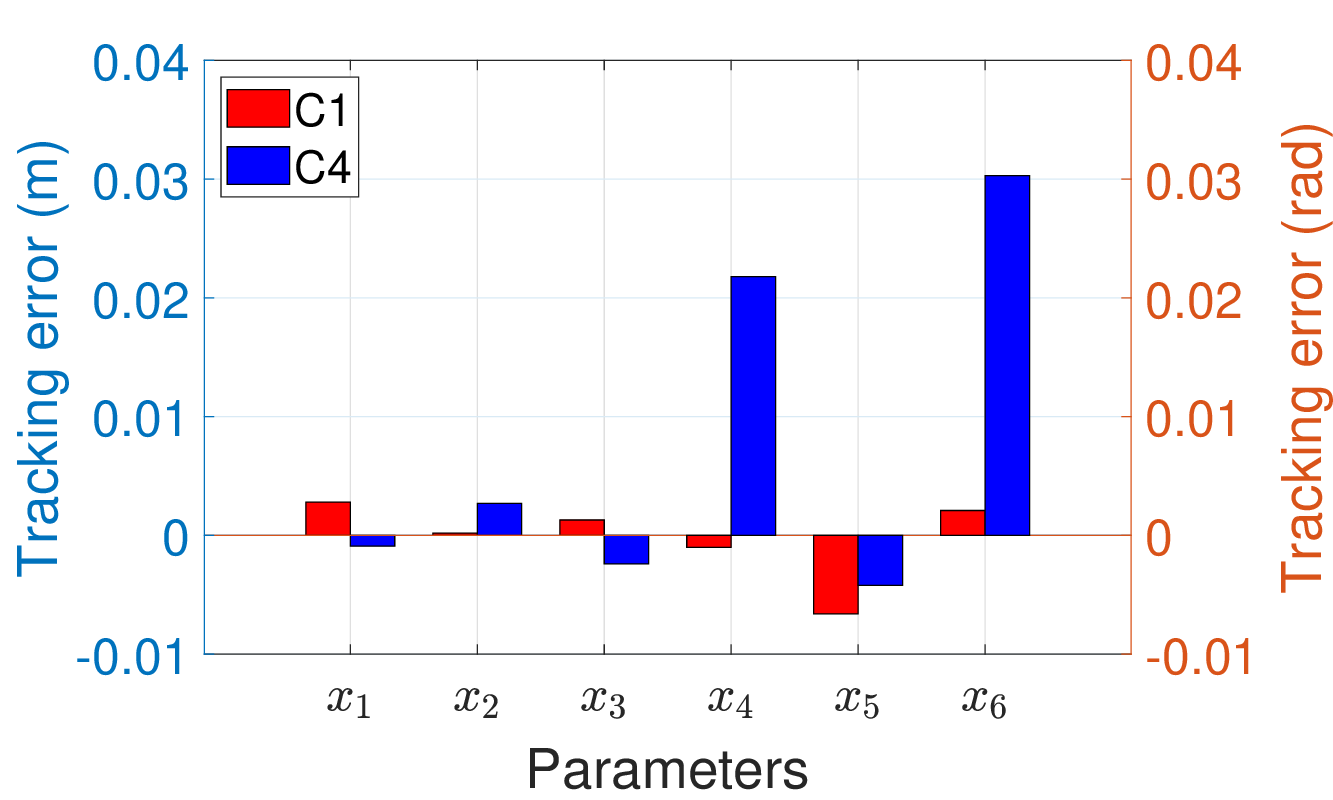}
	\caption{Motion tracking error of C1 and C4 in the experiment in full motion control.}
	\label{exp_mte}
\end{figure}

\begin{table*}[!t]
\renewcommand\arraystretch{1.25}
\centering
\caption{Force tracking performance of two controllers in the experiment. The contents in brackets indicate the percentage improvement compared to the metrics of C4 (the first row).}
\scalebox{0.9}{
\begin{tabular}{cccc}
\toprule
& \multicolumn{3}{c}{\underline{\hbox to 6mm{}Experiment\hbox to 6mm{}}} \\ 
Controller                        & {RMSE}    & {MAE} & {SSE}  \\
\midrule
{C4} & {8.3417} & {2.0913} & {1.4029} \\
{C1} & {\textbf{5.3426 {(35.95\%)}}} & {\textbf{1.9324 {(7.60\%)}}} & {\textbf{-0.0696 {(95.04\%)}}} \\
\bottomrule
\end{tabular}
}
\label{exp_table}
\end{table*}

To further validate the motion/force tracking performance of the proposed approach, experiments are conducted using a mobile manipulator with an Atien TT15 mobile base and a Rokae SR3 manipulator. The experiment platform is shown in Fig. \ref{Experiment}. The end effector is equipped with a 6-axis force sensor, and to replicate the practical working environment, a cleaning tool is mounted at the end effector of the mobile manipulator. In the experiment, the mobile manipulator is tasked to wipe the rigid wall and applies a force of $5N$ against the wall while the mobile base moves along it. The maximum torque output in the experiment is limited to $[6,6,3,2,2,2] N \cdot m$ for safety considerations.

For controller C1, extended UDE parameters are selected as $G_{f1}[i] = \frac{108s}{s^2+8.485s+36}$ and $G_{f2}[i] = \frac{3}{s+3}$, where $i = 1,\cdots,6$. The impedance parameters remain $ K_d = [25, 25, 25, 2.5, 2.5, 2.5]$, $D_d = [10, 10, 10, 1, 1, 1]$, $K_{f,d} = [0, 1, 0, 0, 0, 0]$. At $t=10s$, the values of $f_{d,y}$ are gradually changed by a ramp function from $0N$ to $5N$. After changing the desired force, the mobile base starts to move along the wall at a forward speed of $0.16 m/s$.  

\subsection{Experimental Results and Analysis}
The motion tracking results of C1 and C4 are shown in Fig. \ref{exp_pose} and Fig. \ref{exp_mte}. 
These metrics indicate that C1 exhibits commendable motion tracking capabilities compared to C4 by successfully keeping all tracking errors within 0.01 meters and radians.

The force tracking results are shown in Fig. \ref{exp2_force_result} and Table \ref{exp_table}. For the force tracking performance, notably, C1 almost eliminates the SSE (-0.0696$N$), while the SSE of C4 is 1.4029$N$. During the movement of the mobile base, C4 exhibits oscillations, whereas C1 effectively mitigates this behavior. These observations imply that the proposed approach is proficient in compensating for the dynamic coupling and other unmodeled uncertainties during the mobile base's movement.

Moreover, the RMSE and MAE of C1 improve by approximately 35.95\% and 7.60\% compared to C4. The substantial improvement in RMSE suggests that C1 is particularly effective in reducing the magnitude of more significant errors. The slight improvement in MAE indicates that while C1 also reduces the average error, the improvement is not significant when considering minor deviations. This indicates that C1 is adept at minimizing errors within a smaller range. 

\subsection{Discussion}
The simulation and experiment results reveal that the proposed control approach C1 provides a more robust and effective REI performance for mobile manipulators compared to C2, C3, and C4.
Specifically, it effectively compensates for critical real-world uncertainties, including modeling inaccuracies, friction forces, and dynamic couplings, which usually hinder precise interaction control. This capability is evident from the reduction in tracking errors from C4 to C1, as shown in Table \ref{sim_table} and Table \ref{exp_table}.
 
However, it is important to note that environment stiffness and damping were set to mimic compliant materials in the simulations. Under such conditions, the equilibrium point of interaction between the mobile manipulator and the environment shifts. This is evident in Fig. \ref{sim1_motion} (b) and Fig. \ref{sim2_motion} (b), where the position $x_2$ of the end effector exceeds $x_{2,d}$, and a noticeable deviation of the actual force occurs from the tracked value in Fig. \ref{sim1_force} and Fig. \ref{sim2_force}. Conversely, in the experiment, the mobile manipulator is tasked to wipe the rigid wall, and the environment parameters are more substantial, ensuring that the force trajectory remains closely aligned with the set course.
This reveals a limitation of the proposed approach: Deviations become pronounced when interacting with soft environments, where impedance mismatches lead to force control errors. Integrating adaptive impedance strategies can further enhance interaction stability and precision, especially for applications involving compliant or deformable surfaces.

\section{Conclusion}
In this paper, a dynamic motion/force control approach via extended UDE is proposed for the dynamic interaction of mobile manipulators. The extended UDE is designed to estimate dynamic coupling terms and other unmodeled uncertainties and incorporate them into the feedforward and feedback control loops, respectively.
The series of simulations and experiments of a wall-cleaning task validate the effectiveness of the proposed control approach. Ablation studies demonstrate that, with the dynamic coupling-integrated model, the proposed extended UDE significantly improves the motion/force tracking performance when the mobile base is in dynamic motion. 
Future work will focus on enhancing the system's interaction capabilities by integrating the perception of unknown environments and improving adaptability for interactions with compliant surfaces.

\section*{Declaration of competing interest}
The authors declare that they have no known competing financial interests or personal relationships that could have appeared to
influence the work reported in this paper.

\section*{Acknowledgment}
This work was supported by the National Natural Science Foundation, China (No.62088101)
; the InnoHK of the Government of the Hong Kong SAR via the Hong Kong Centre for Logistics Robotics.

\printcredits

\bibliographystyle{cas-model2-names}
\bibliography{main}

\end{document}